\documentclass[twoside,11pt]{article}

\usepackage{jmlr2e}
\usepackage[ruled,linesnumbered]{algorithm2e}
\usepackage{amsmath}
\usepackage{bm}
\usepackage{subcaption}
\usepackage{float}

\jmlrheading{21}{2020}{1-48}{4/20}{/}{Jingfan Chen, Guanghui Zhu, Chunfeng Yuan, and Yihua Huang}

\ShortHeadings{Semi-supervised Embedding Learning for High-dimensional Bayesian Optimization}{Chen, Zhu, Yuan, and Huang}
\firstpageno{1}

\begin{document}

\title{Semi-supervised Embedding Learning for High-dimensional Bayesian Optimization }

\author{\name Jingfan Chen \email jingfan.chen@smail.nju.edu.cn 
       \AND
       \name Guanghui Zhu\thanks{Corresponding authors with equal contribution} \email guanghui.zhu@smail.nju.edu.cn 
       \AND
       \name Chunfeng Yuan \email cfyuan@nju.edu.cn   
       \AND
       \name Yihua Huang\footnotemark[1] \email yhuang@nju.edu.cn \\
       \addr State Key Laboratory for Novel Software Technology\\
       Nanjing University \\
       Nanjing 210023, China
       }

\editor{}

\maketitle

\begin{abstract}%   <- trailing '%' for backward compatibility of .sty file
Bayesian optimization is a broadly applied methodology to optimize the expensive black-box function. Despite its success, it still faces the challenge from the high-dimensional search space. To alleviate this problem, we propose a novel Bayesian optimization framework, which finds a low-dimensional space to perform Bayesian optimization through a semi-supervised, iterative, and embedding learning-based method~(SILBO). SILBO incorporates both labeled and unlabeled points acquired from the acquisition function of Bayesian optimization to guide the learning of the embedding space. To accelerate the learning procedure, we present a randomized method for generating the projection matrix. Furthermore, to map from the low-dimensional space to the high-dimensional original space, we propose two mapping strategies: $\text{SILBO-BU}$ and $\text{SILBO-TD}$ according to the evaluation overhead of the objective function. Experimental results on both synthetic function and hyperparameter optimization tasks demonstrate that SILBO outperforms the existing state-of-the-art high-dimensional Bayesian optimization methods. Meanwhile, the semi-supervised dimensional reduction and iterative learning in SILBO are effective for high-dimensional Bayesian optimization.
\end{abstract}

\begin{keywords}
  Bayesian Optimization, High-dimensional Optimization, Semi-supervised Dimension Reduction, Embedding Learning
\end{keywords}

\section{Introduction}

As a well-established approach for sample-efficient global optimization of black-box functions that are expensive to evaluate, Bayesian optimization~(BO) is used in many tasks such as hyperparameter tuning \citep{Hutter2011Sequential,Bergstra2011Algorithm,Snoek2012Practical}, neural architecture search \citep{Kandasamy2018Neural}, and chemical structure search \citep{Bombarelli2018Automatic}. BO provides a principled method for finding the global optimum of black-box function: using the cheap probability surrogate model of black-box function as the input to the acquisition function, repeatedly considering the trade-off between exploitation and exploration to select the promising points. The surrogate model is constructed based on the evaluation values observed so far. A widely-used surrogate model is Gaussian Process regression, which provides the uncertainty quantification of the function value by imposing a Gaussian Process prior.

While BO provides such an automated procedure, it still faces a huge challenge in high-dimensional scenarios. To ensure the converge for learning the function response value, the sample complexity depends exponentially on the number of dimensions \citep{Shahriari2016Taking}. In practice, BO is limited to the optimization problem with around 20 dimensions when using Gaussian Process regression as a surrogate model \citep{frazier2018tutorial}.

To handle high-dimensional Bayesian optimization, many methods have been proposed. Based on the assumption that only a small number of dimensions influence the response value of the objective function, the embedding methods perform BO on a low-dimensional space. The corresponding projection matrix can be constructed randomly \citep{Wang2013Bayesian,nayebi19a}, or learned actively \citep{Djolonga2013High,zhang2019high}. Some methods impose an additive structure on the objective function \citep{gardner17Discovering,k2015high}. Besides, many methods start from the way of learning low-dimensional embedding and find an effective subspace through nonlinear dimension reduction \citep{lu2018Structured}. However, there are two limitations in the existing methods. First, the projection matrix is immutable. Once the generated low-dimensional embedding cannot represent the intrinsic structure of the objective function, finding the global optimum through Bayesian optimization will become very difficult. Second, the low-dimensional space is learned in a supervised way. The label of each point in the learning-based dimensional reduction method indicates the response value of the black-box function. To learn an effective low-dimensional space, a large number of labeled points are required, which leads to huge computation costs especially when the evaluation overhead of the objective function is expensive.

In this paper, we propose a novel framework called SILBO\footnote{SILBO stands for \textbf{S}emi-supervised, \textbf{I}terative, and \textbf{L}earning-based \textbf{B}ayesian \textbf{O}ptimization. The code is
available at \url{https://github.com/cjfcsjt/SILBO.}} to mitigate the problem of \textit{the curse of dimensionality} by learning the effective low-dimensional space iteratively through the semi-supervised dimension reduction method. After a low-dimensional space is identified, Bayesian optimization is performed in the low-dimensional space, leading to a stable and reliable estimation of the global optimum. Specifically, the contribution of this paper is as follows:
\begin{itemize}
\item We propose an iterative method in SILBO to update the projection matrix dynamically. During each iteration of BO, a semi-supervised low-dimensional embedding learning method is proposed to construct the projection matrix by utilizing both labeled and unlabeled points acquired from the acquisition function of BO.
\item To accelerate the semi-supervised dimension reduction, we further propose a randomized method to compute the high-dimensional generalized eigenvalue problem efficiently. We also analyze its time complexity in detail.

\item Furthermore, to map from the low-dimensional space to the high-dimensional original space, we propose two mapping strategies: $\text{SILBO-BU}$ and $\text{SILBO-TD}$ according to the evaluation overhead of the objective function.

\item Experimental results on both synthetic function and neural network hyperparameter optimization tasks reveal that SILBO outperforms the existing state-of-the-art high-dimensional Bayesian optimization methods. Meanwhile, the semi-supervised dimensional reduction and iterative learning in SILBO are effective for high-dimensional Bayesian optimization.
\end{itemize}

The rest of this paper is organized as follows. Section \ref{relatedwork} gives an overview of related work. Section \ref{preliminary} states the problem and lists relevant background materials. The SILBO algorithm is proposed in Section \ref{algorithm}. 
%
%and the details of the algorithm are illustrated in Section \ref{semisubspace}-\ref{mapping}. 
The experimental evaluation is presented in Section \ref{numericalresults} and the conclusion is given in Section \ref{conclusion}.

% related-work
\section{Related Work}\label{relatedwork}
Bayesian optimization has achieved great success in many applications with low dimensions \citep{k2017multifidelity,klein2016fast,Swersky2013Multi,wu2019practical,Wu2017Bayesian,hernndezlobato2015predictive}. However, extending BO to high dimensions is still a challenge. Recently, the high-dimensional BO has received increasing attention and a large body of literature has been devoted to addressing this issue.

Given the assumption that only a few dimensions play a decisive role, the linear low-dimensional embedding method achieves dimension reduction using a projection matrix. In REMBO \citep{Wang2013Bayesian}, the projection matrix is randomly generated according to Gaussian distribution. The promising points are searched in the low-dimensional space by performing Gaussian Process regression and then mapped back to the high-dimensional space for evaluation. It has been proven that REMBO has a great probability to find the global optimum by convex projection, although the high probability is not guaranteed when the box bound exists. Another problem of REMBO is the over-exploration of the boundary. To address the over exploration, a carefully-selected bound in the low-dimensional embedding space was proposed \citep{binois2017choice}, which finds the corresponding points in the high-dimensional space by solving a quadratic programming problem. The BOCK algorithm \citep{oh2018bock} scales to the high-dimensional space using a cylindrical transformation of the search space. HeSBO \citep{nayebi19a} employs the count sketch method to alleviate the over-exploration of the boundary and use the hash technique to improve  computational efficiency. HeSBO also shows that the mean and variance function of Gaussian Process do not deviate greatly under certain condition. However, the above-mentioned methods only use the prior information to generate the projection matrix randomly and do not employ the information of the actual initial points to learn a low-dimensional embedding actively. 

Different from the previous methods, the learning-based methods have been proposed. SIRBO \citep{zhang2019high} uses the supervised dimension reduction method to learn a low-dimensional embedding, while SI-BO \citep{Djolonga2013High} employs the low-rank matrix recovery to learn the embedding. However, the low-dimensional embedding learned by these methods is immutable. Once the projection matrix is generated according to the initial points, it will not be updated. In some scenarios, because of the small number of initial points that have been evaluated, the low-dimensional embedding space cannot accurately reflect the information of the objective function.

Another way for handling the high-dimensional BO is to assume an additive structure \citep{gardner17Discovering} of the objective function. Typically,  ADD-BO \citep{k2015high} optimizes the objective function on a disjoint subspace decomposed from the high-dimensional space. Unfortunately, the additive assumption does not hold in most practical applications. Besides, the non-linear embedding method is also attractive \citep{eissman2018bayesian,lu2018Structured,moriconi2019highdimensional}. These non-linear methods use the Variational Autoencoder~(VAE) to learn a low-dimensional embedding space. The advantage of non-linear learning methods is that points in the original space can be easily reconstructed through the non-linear mapping. However, training VAE requires a large number of points. When the evaluation cost of the objective function is expensive, the non-linear embedding method is almost impractical.

In this paper, we focus on the linear low-dimensional embedding method and propose an iterative, semi-supervised method to learn the embedding.

\section{Preliminary}\label{preliminary}

In this section, we give the problem setup and introduce Bayesian optimization~(BO), semi-supervised discriminant analysis~(SDA), and slice inverse regression~(SIR).

\subsection{Problem Setup}
Consider the black-box function $f:\mathcal{D} \rightarrow [0,1]$, defined on a high-dimensional $d$ and continuous domain $\mathcal{D} = [-1,1]^d \subset \mathbb{R}^d$. $f(\bm{x})$ is computationally expensive and may be non-convex. Given $\bm{x} \in \mathcal{D} $, we can only access the noisy response value $y$ extracted from $f(\bm{x})$ with noise $\epsilon \sim \mathcal{N}(0,\sigma^2)$. Also, we assume that the objective function contains $r \leqslant d$ intrinsic dimensions. In other words, given an embedding matrix $B \in \mathbb{R}^{r \times d}$  with orthogonal rows and a function $g:\mathbb{R}^r \rightarrow [0,1]$, $f(\bm{x}) = g(B\bm{x})$. Our goal is to find the global optimum.
\begin{equation}
    \bm{x}^* = \mathop{\arg\max}\limits_{\bm{x}\in \mathcal{D}}f(\bm{x})
\end{equation}

\subsection{Bayesian Optimization}
Bayesian optimization is an iterative framework, which combines a surrogate model of the black-box function with a search strategy that tries to find possible points with large response values. Given $t$ observation points $\bm{x}_1,...,\bm{x}_t \in \mathcal{D}$ and their corresponding evaluation values $y_1,...,y_t$, the surrogate model is usually Gaussian Process regression that imposes a prior, $f(\bm{x}_{1:t}) \sim \mathcal{GP}(\mu(\bm{x}_{1:t}), k(\bm{x}_{1:t},\bm{x}_{1:t}))$, to the objective function with mean function $\mu$ at each $\bm{x}_i$ and covariance function or kernel $k$ at each pair of points $(\bm{x}_i,\bm{x}_j)$. The kernel function describes the similarity between inputs. One of the widely-used kernel functions is the Matérn kernel. Then, given a new point $\bm{x}^*$, the prediction of the response value can be calibrated by the posterior probability distribution~(noise-free).

\begin{eqnarray}
&&f(\bm{x}^*)|f(\bm{x}_{1:t}) \sim \mathcal{N}(\mu_t(\bm{x}^*),\sigma_t(\bm{x}^*))\label{gp}\\
&&\mu_t(\bm{x}^*) = k(\bm{x}^*,\bm{x}_{1:t})k(\bm{x}_{1:t},\bm{x}_{1:t})^{-1}(f(\bm{x}_{1:t})-\mu(\bm{x}_{1:t}))+\mu(\bm{x}^*)\\
&&\sigma_t(\bm{x}^*) = k(\bm{x}^*,\bm{x}^*) - k(\bm{x}^*,\bm{x}_{1:t})k(\bm{x}_{1:t},\bm{x}_{1:t})^{-1}k(\bm{x}_{1:t},\bm{x}^*)
\end{eqnarray}

At each iteration of Bayesian optimization, the predictive mean and variance are regarded as uncertainty quantification, supporting the subsequent acquisition function optimization. The acquisition function tries to balance between exploration~(high variance) and exploitation~(high mean value). The commonly-used acquisition function for searching promising points is UCB~(Upper Confidence Bound) \citep{Srinivas2010Gaussian}.
UCB tries to select the next point with the largest plausible response value according to Equation~\ref{ucb}.
\begin{equation}
    \bm{x_{t+1}} = \mathop{\arg\max}\limits_{\bm{x}\in\mathcal{D}}\mu_t(\bm{x})+\beta_t^{1/2}\sigma_t(\bm{x})
    \label{ucb}
\end{equation}
where $\beta_t$ is a parameter set used to achieve the trade-off between exploration and exploitation. In this work, we also experiment with EI~(Expected Improvement) \citep{Snoek2012Practical}, which is another popular acquisition function.

\subsection{Semi-supervised Discriminant Analysis}
Semi-supervised discriminant analysis~(SDA) \citep{Cai2007Semi} is a semi-supervised linear dimension reduction algorithm that leverages both labeled and unlabeled points. SDA aims to find a projection that reflects the discriminant structure inferred from the labeled data points, as well as the intrinsic geometrical structure inferred from both labeled and unlabeled points. SDA is an extension of linear discriminant analysis~(LDA). The original LDA aims to find a projection $\beta$ such that the ratio of the between-class scatter matrix $S_b$ to the total within-class scatter matrix $S_t$ is maximized. When the number of training data is small, it can easily lead to overfitting. SDA solves this problem by introducing a regularizer $J(\beta)$ combined with unlabeled information.
\begin{equation}
    \beta^* = \arg\max_\beta  \frac{\beta^{\top}S_b\beta}{\beta^{\top}S_t\beta + \alpha J(\beta)} \label{jbeta}
\end{equation}
where $\alpha$ is a coefficient used to balance between model complexity and experience loss.

$J(\beta)$ is constructed by considering a graph $S$ incorporating neighbor information of the labeled points and the unlabeled points, where $S_{ij}$ indicates whether $x_i$ and $x_j$ are neighbors. Motivated from spectral dimension reduction \citep{Belkin2002Laplacian}, the regularizer can be defined as $J(\beta) = \Sigma_{ij}(\beta^{\top} x_i-\beta^{\top} x_j)^2S_{ij}$ for any two points $x_i$ and $x_j$. Then, given the dataset $X$, $J(\beta)$ can be written as:
\begin{equation}
    J(\beta) = 2\beta^{\top}X(D-S)X^{\top}\beta =  2\beta^{\top}XLX^{\top}\beta \label{transform}
\end{equation}
where $D$ is a diagonal matrix, $D_{ii} = \sum_{j}S_{ij}$, and $L$ is a Laplacian matrix \citep{Chung1997Spectral}.
Finally, SDA can be reformulated as solving the following generalized eigenvalue problem.
\begin{equation}
    S_b\beta= \lambda (S_t+\alpha XLX^{\top}) \beta \label{SDA}
\end{equation}
% \begin{equation}
%     \hat{I} =\begin{bmatrix}
%  I&0 \\ 
%  0& 0
% \end{bmatrix}
% \end{equation}
% where $I$ is a $n_l \times n_l$ identity matrix. 

\subsection{Slice Inverse Regression}
Sliced inverse regression~(SIR) \citep{Li1991Sliced} is a supervised dimension reduction method for continuous response values. SIR aims to find an effective low-dimensional space. The dimension reduction model is:
\begin{equation}
    Y = g(\beta^{\top}_1\bm{x},...,\beta_r^{\top}\bm{x},\epsilon)\label{reductionmodel}
\end{equation}
Here, $Y$ is the response variable, $\bm{x} \in \mathbb{R}^d$ is an input vector, $g$ is an unknown function with $r+1$ arguments. $\{ 
\beta_1,\cdots,\beta_r\}$ denotes orthogonal projection vectors, $r$ denotes the dimensions of the $e.d.r.$~(effective dimension reducing) space and $\epsilon$ is noise. The core idea of SIR is to swap the positions of $\bm{x}$ and Y. The algorithm cuts response value $Y$ into $H$ slices and consider the $H$-dimensional inverse regression curve $E(\bm{x}|Y) = (E(\bm{x}_1|Y),...,E(\bm{x}_H|Y))$ rather than regressing $Y$ on $\bm{x}$ directly. SIR assumes the existence of $e.d.r.$ directions, and the curve that just falls into an $r$-dimensional space. SIR finds the $e.d.r.$ directions by minimizing the total within-slice scatter $\Sigma_x$ and maximize the between-slice scatter $\Sigma_\eta$. Similar to LDA, the problem can be reformulated as a generalized eigenvalue problem.
\begin{equation}
    \Sigma_\eta \beta = \lambda\Sigma_x \beta \label{generalizedproblem}
\end{equation}
Give the assumption that only a small number of dimensions influence the response value of the objective function, SIR can find an effective low-dimensional space without losing the essential information to predict response values, regardless of whether the points are independent and identically distributed (i.i.d.).
SIR can be easily extended to support semi-supervised dimension reduction. For example, Semi-SIR  \citep{Wu2010Local} and Semi-KSIR \citep{Huang2014Semi} take the unlabeled data into consideration.

%Algorithm
\section{The SILBO Algorithm}\label{algorithm}

\subsection{Overview}\label{overview}

In this section, we propose a framework that addresses the high-dimensional optimization challenge by learning a low-dimensional embedding space $\mathcal{Z}$ associated with a projection matrix $B$. To learn the intrinsic structure of the objective function effectively, we iteratively update $B$ through semi-supervised dimension reduction. Moreover, we propose a randomized method to accelerate the computation of the embedding matrix $B$. 
%in the high-dimensional scenario.
%
By performing BO on the learned low-dimensional space $\mathcal{Z}$, the algorithm can approach the $\bm{z}^*\in \mathcal{Z}$ that corresponds to the optimum $\bm{x}^* \in \mathcal{D}$ as close as possible.

\begin{figure}
	\centering
	\begin{subfigure}[t]{\textwidth}
		\centering
		\includegraphics[width=5in,height =1.9in]{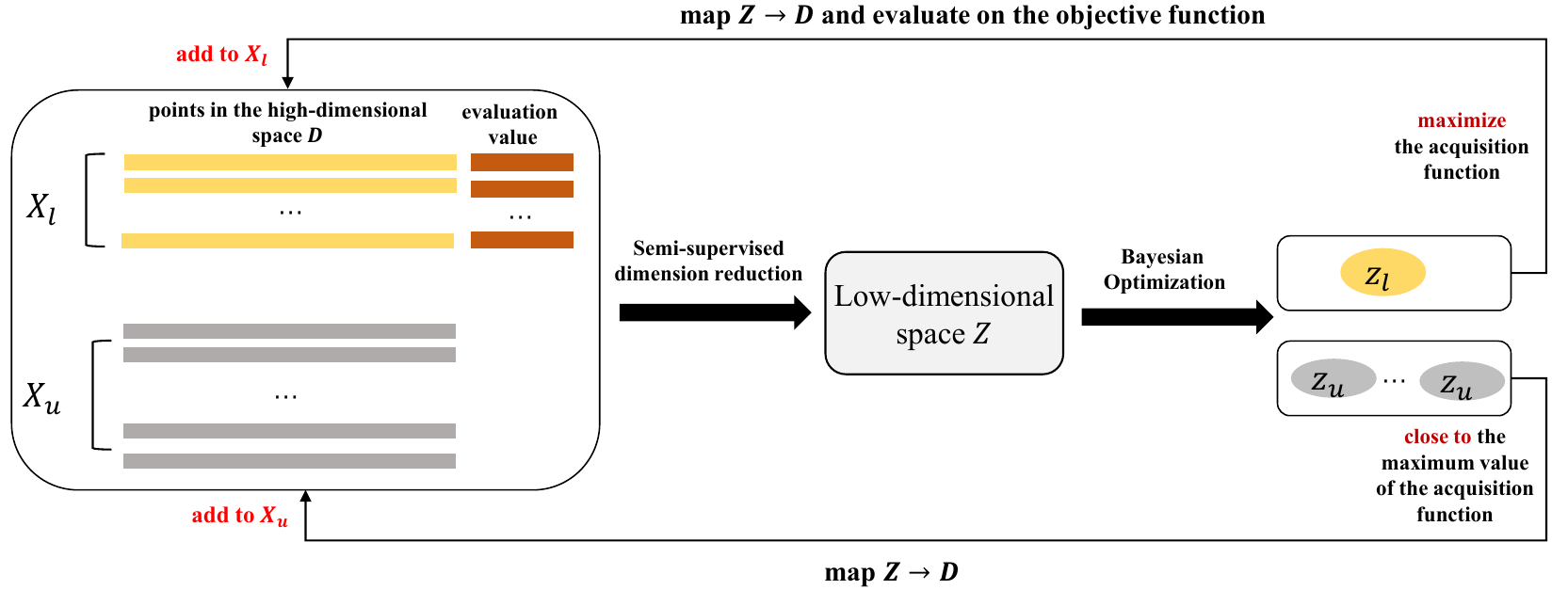}
		\label{fig:overview}
		\vspace{-1ex}
	\end{subfigure}
	
	\caption{An overview of SILBO.}
	\label{fig:overview}
	\vspace{-4ex}
\end{figure}

As shown in Figure \ref{fig:overview}, given the labeled points $X_l$ and unlabeled points $X_u$ where the label represents the evaluation value of the corresponding point, SILBO consists of three tightly-connected phases:
\begin{itemize}
    \item semi-supervised embedding learning, which tries to find an effective low-dimensional space $\mathcal{Z}$ of the objective function by utilizing both $X_l$ and $X_u$;
    \item performing Gaussian Process regression on the learned low-dimensional embedding space and selecting candidate points according to the acquisition function of BO;
    \item evaluating points by the objective function $f$, then updating the Gaussian Process surrogate model and the projection matrix $B$.
\end{itemize}

\begin{algorithm}[htb]
	\caption{SILBO}\label{algorithm1}
	\KwIn{Objective function $f$, acquisition function $\alpha$, high dimension $d$, effective dimension $r$}
	\KwOut{The optimum $x^*$}
	Initialize samples $X_u^0,X_l^0$ with size $n_u$ and $n_l$\;
	Construct $B_0$ with Algorithm \ref{algorithm2}\;
	Let $D_0 = \{B_0X^0_l,f(X^0_l)\}$\;
	Construct Gaussian Process regression model based on $D_0$\;
	$Z_l^0=\varnothing$,$Y_0=\varnothing$\;
	\For{$t=1$ to $N$}{
		$Z_u^{t} = \varnothing$\;
		Generate random point set $C$ from the low-dimensional space\;
		$z_l = \arg\max_{z \in C} \alpha (z)$\;
		$Z_l^{t} = Z_l^{t-1}\cup z_l$\;
		$C = C - z_l$\;
		\For{$n=1$ to $n_u$}{
			$z_u$ = $ \arg\max_{z \in C}\alpha (z)$\;
			$Z_u^{t} \cup z_u$\;
			$C = C - z_u$\;}
		Construct $X_l^{t}$,$X_u^{t}$,$D_{t}$ based on $Z_l^{t}$,$Z_u^{t}$\;
		Update Gaussian process regression model based on $D_{t}$\;
		Update $B_{t}$ with Algorithm \ref{algorithm2} based on $X_u^{t}$ , $X_l^{t}$ and $Y_{t}$\;
	}
\end{algorithm}

These phases are also summarized in Algorithm \ref{algorithm1}.
The first step is to construct the projection matrix $B$ and find an effective low-dimensional space that can keep the information of the original space as much as possible. Both $X_l$ and $X_u$ are selected according to the acquisition function of BO. $X_l$ is the labeled points that have been evaluated on the objective function. $X_u$ is the unlabeled points whose evaluation values are unknown.

The second step is to find the possible low-dimensional optimum $\bm{z}_l \in \mathcal{Z}$ which maximizes the acquisition function $\alpha$ and select several promising unlabeled points $\bm{z}_u \in \mathcal{Z}$ that can be used to update $B$ in the next iteration. Then, $\bm{z}_l$ and $\bm{z}_u$ are mapped back to the high-dimensional space $\mathcal{D}$ through a specific mapping $h:\mathcal{Z} \rightarrow \mathcal{D}$ to get $\bm{x}_l$ and $\bm{x}_u$.

Finally, we compute $y$ by evaluating the objective function $f$ on $x_l$ and update the GP model by $(z_l,y)$. The $x_l$ and $x_u$ will be added to $X_l$ and $X_u$ respectively for updating the embedding in the next iteration.

The low-dimensional embedding is learned through SIR combined with the semi-supervised technique.
The between-slice scatter matrix $\Sigma_\eta$ and total within-slice scatter matrix $\Sigma_x$ are constructed by utilizing the local information as well as the unlabeled points. Then, $B$ is obtained through solving the generalized eigenvalue problem. Based on the randomized SVD method, Algorithm \ref{algorithm2} is proposed to speed up the solution to this problem. Moreover, we carefully analyze the mapping $h$ between the low-dimensional space and the high-dimensional space and further propose two strategies for the scenarios with different evaluation costs.

%semi
\subsection{Semi-supervised Dimension Reduction}\label{semisubspace}

The assumption is that there is a low-dimensional space that preserves the information of the objective function $f(\bm{x})$ defined in the high-dimensional space. In other words, the dimensionality of $\mathcal{D}$ can be reduced without losing the essential information to predict response values $Y$. If there are enough evaluated points for the initialization of Bayesian optimization, we may be able to explore the intrinsic structure of $f(\bm{x})$ through these points. However, for optimization problems with high computational cost, only a few evaluated points are available. Thus, it is difficult to learn proper embedding only through them. 
%This is also proved empirically in Section \ref{numericalresults}.
%
In such a case, reasonable and effective use of unevaluated points acquired from the acquisition function of BO will be helpful for embedding learning.

Although SDA provides an effective strategy to incorporate the unlabeled data, it is only suitable for classification problems. Therefore, we need to extend it to the scenarios where the response value is continuous. In fact, SIR is equivalent to LDA, and SDA is an extension of LDA as a generalized eigenvalue problem \citep{Wu2010Local}. Thus, SDA can be applied to regression setting by taking advantage of the discreteness brought by the slicing technique in SIR. We call such extension semi-SIR. Next, we introduce how to construct semi-SIR to learn a low-dimensional embedding.

First, the evaluated points are sorted and cut into $H$ slices according to their response values. Then, similar to SDA \citep{Cai2007Semi}, our problem can be reformulated as follows. Given labeled points $X_l \in \mathbb{R}^{n_l \times d}$ and unlabeled points $X_u \in \mathbb{R}^{n_u \times d}$, $n_l$ and $n_u$ denote the number of the labeled and unlabeled points respectively, we aim to find the embedding matrix $B$ through solving:
\begin{equation}
    X^{\top}WX\beta= \lambda X^{\top}(\hat{I}+\alpha L)X \beta\label{semisir}
\end{equation}
where $X \in \mathbb{R}^{n \times d}$ denotes a centered data matrix whose rows represent $n$ samples in $\mathbb{R}^d$, $n= n_u + n_l$. $W$ is a $n \times n$ weight matrix. $\hat{I}$ can be expressed as: 
 \begin{equation}
     \hat{I} =\begin{bmatrix}\label{identitymatrix}
 I&0 \\ 
  0& 0
\end{bmatrix}
\end{equation}
where $I$ is a $n_l \times n_l$ identity matrix.

Note that the between-slice scatter matrix $X^{\top}WX$ can be written in an elegant linear algebraic formulation.
\begin{equation}
    X^{\top}WX = X^{\top}\Omega\Omega^{\top}X= X^{\top}\begin{bmatrix}
 \Omega_{l}\Omega^{\top}_{l}&0 \\ 
 0& 0
\end{bmatrix}X \label{elegantform}
\end{equation}
where $\Omega_l \in \mathbb{R}^{n_l \times H}$ denotes the rescaled slice membership matrix for those evaluated samples with $\Omega_{l_{ij}} = 1/\sqrt{n_j}$ if the $i$-th sample of $X_l$ is a member of the $j$-th slice. Otherwise, $\Omega_{l_{ij}}=0$. $n_j$ denotes the number of the samples in the $j$-th slice.
% Therefore, like \citep{Cai2007Semi}, our goal can be formulated as, for any two samples $x_i$,$x_j$, we want to find a direction $\beta$ that minimizes the distance between their low-dimensional representations:
% \begin{equation}
%     \min_{\beta} \Sigma_{ij}(\beta^{\top} x_i-\beta^{\top} x_j)^2S_{ij} = \min_{\beta} 2\beta^{\top} X^{\top}LX\beta
% \end{equation}
% where $L$ is a Laplacian matrix which can be computed as $L = D-S $. $S$ denotes the symmetric affinity matrix:
% \begin{equation}
%     S_{ij} = \left\{\begin{matrix}
%  1& \text{if} x_i \in N_k(x_j)\text{or} x_j \in N_k(x_i) \\ 
%  0& \text{otherwise}
% \end{matrix}\right.
% \end{equation}
% and $D$ is diagonal matrix where $D_{ii} = \sum_{j}S_{ij}$.

% At the same time, for a small number of evaluated samples, we not only know the distance information, but also have information of their response values and corresponding slice structures. 
% So we want to find a direction $\beta$ which makes the between-slice distance as large as possible, and the within-slice distance as small as possible. This is consistent with the objectives of the original SIR.
For the labeled points in the same slice, their response values are very close, but this closeness may not be retained in $\mathcal{D}$. There could be a small number of points that are far away from others in each slice. Although these \textit{outliers} may indicate that there exist other areas with large response values of the objective function, they are likely to interfere with the embedding we have learned. Thus, to reveal the main information of the objective function as much as possible, we employ the localization technique. By strengthening the local information of each slice, the degeneracy issue in the original SIR can be mitigated \citep{Wu2010Local}. Next, we  illustrate how to construct $W$ in Equation \ref{elegantform} with the local information.

We note that $\Omega_l$ in Equation \ref{elegantform} is a block matrix, and each block corresponds to a slice.
\begin{equation}
    \Omega_l = \begin{pmatrix}
\Omega_l^1 &  & \\ 
 & \ddots & \\ 
 &  & \Omega_l^H 
\end{pmatrix}\label{omega}
\end{equation}

For the original between-slice scatter $X^{\top}WX$, $\Omega_l$ can only indicate whether a sample point belongs to a slice. Here, we will strengthen the localization information by introducing a localized weight. For those evaluated samples, we let $\Omega^h_{ij} = 1/k_h$ if the $i$-th sample belongs to the $k$-nearest neighbors of the $j$-th sample~($j$ can equal to $i$) in the $h$-th slice. Otherwise, $\Omega^h_{ij} = 0$. $H$ denotes the number of slices, $k$ is a parameter for $k$NN, and $k_h$ is the number of neighbor pairs in the $h$-th slice. 

In summary, we aim to estimate the intrinsic structure of the objective function and find an effective subspace that reflects the large response value. Then, we are in a better position to identify the possible global optimum. Therefore, in combination with Bayesian optimization whose acquisition function will provide us candidate points with potentially large response values, semi-supervised dimension reduction can alleviate overfitting caused by the small size of labeled points. Thus, by utilizing both the labeled and unlabeled points, we can learn an $r$-dimensional space that preserves the intrinsic structure of the objective function better. 
%random 
\subsection{Randomized Semi-supervised Dimension reduction}\label{random}

In the high-dimensional scenario, learning a low-dimensional embedding space still requires a lot of time. The main bottleneck lies in the solution of generalized eigenvalue problem in Equation \ref{semisir} and the computation of the neighbor information in each slice. The major computation cost of the traditional method such as the two-stage SVD algorithm is the decomposition of the large-scale matrix. Fortunately, the randomized SVD~(R-SVD) technique \citep{Halko2011Finding} shed a light on this problem.

The contribution of R-SVD here is that when the effective dimension $r \ll n < d$, using randomized SVD to approximate a $n \times d$ matrix only requires $O(ndr)$, rather than $O(n^2d)$. Moreover, the empirical results find that when solving the generalized eigenvalue problem, the performance of the randomized algorithm is superior to that of the traditional deterministic method in the high-dimensional case \citep{georgiev2012randomized}. Thus, we first replace the two-stage SVD with R-SVD to accelerate the solution of Equation \ref{semisir}.

We note that the decomposition overhead of the $X^T\Omega \in \mathbb{R}^{d \times n}$ in Equation \ref{semisir} is huge when $d$ is large. Thus, we decompose $X^{\top}\Omega = U_1S_1V_1^{\top}$ using the R-SVD algorithm. Then, the between-slice scatter matrix can be expressed as:
\begin{equation}
    X^{\top}\Omega\Omega^{\top}X = U_1S_1^2U_1^{\top}\label{betweenslice}
\end{equation}

Similarly, due to the symmetry of matrix $\hat{I} + \alpha L$, the right hand side of Equation \ref{semisir} can be decomposed as:
\begin{equation}
    X^{\top}(\hat{I} + \alpha L)X  = X^{\top}MM^{\top} X\label{withinslice}
\end{equation}
where $M$ can be constructed through Cholesky decomposition.
Thus, the generalized eigenvalue problem in Equation \ref{semisir} can be reformulated as:
\begin{equation}
    \frac{1}{\lambda} e= S_1^{-1}U_1^{\top}X^{\top}M M^{\top}XU_1S_1^{-1}e\label{finalform}
\end{equation}
where $e = S_1U_1^{T}\beta$.

Next, if we let $A = S_1^{-1}U_1^{\top}X^{\top}M $, then we can see that this is a traditional eigenvalue problem. $A$ is an $r \times n$ matrix, $r$ is small enough and thus we can decompose it through original SVD naturally: $A = U_2S_2V_2^{\top}$. Finally, the embedding matrix can be computed as:
\begin{equation}
 B = (U_1S_1^{-1}e_1,\cdots,U_1S_1^{-1}e_r) = U_1S_1^{-1}U_2   \label{matrixb}
\end{equation}

The original time complexity of the construction of $X^T\Omega\in \mathbb{R}^{d \times n} $ and $X^TM\in \mathbb{R}^{d \times n} $ is $O(d^2nn_h)$ and $O(d^2nk)$ respectively, including the computation of the exact $k$NN for each point, where $n_h$ is the number of samples in each slice. In addition, the time complexity of decomposition of $X^T\Omega$ is $O(n^2d)$ and the time complexity of construction and computation of Equation \ref{finalform} is $O(ndr)$. Therefore, the overall time complexity is $O(d^2n\max(n_h,k)+n^2d+ndr)$, which is prohibitive in the high-dimensional case. 

To alleviate the expensive overhead in high-dimensional case, we use the fast $k$NN \citep{2003Database,Ailon2009the} to compute the neighbor information, we find that constructing $X^T\Omega$ and $X^TM$
only requires $O(tdnn_h)$ and $O(tdnk)$ respectively. Moreover, using R-SVD to factorize the matrix $X^T\Omega$ only requires $O(ndr)$. Therefore, The overall time complexity reduces to $O(tdn\max(n_h,k)+ndr)$. $t$ denotes a fixed parameter about the logarithm of $n$.

Algorithm \ref{algorithm2} summarizes the above steps.

\begin{algorithm}
\caption{Randomized semi-supervised dimension reduction}\label{algorithm2}
\KwIn{$X_l\in \mathbb{R}^{n_l \times d}$, $X_u\in \mathbb{R}^{n_u \times d}$, number of slice $H$; number of nearest neighbor $k$, effective dimension $r$}
\KwOut{$B \in \mathbb{R}^{d \times r}$} 

compute $X^{\top}\Omega$, $ X^{\top}M$ from Equation \ref{elegantform} and Equation \ref{withinslice}\;
estimate $[U_1,S_1,V_1]$ = Randomized SVD($X^{\top}\Omega$,$r$)\;
$A = S_1^{-1}U_1X^{\top}M$\;
factorize $[U_2,S_2,V_2] $= SVD($A$)\;
compute $B$ from Equation \ref{matrixb}

\end{algorithm}
%mapping from low to high
\subsection{Mapping From Low to High}\label{mapping}

After semi-supervised dimensional reduction, we perform Bayesian optimization with GP on the low-dimensional space $\mathcal{Z}$. But we need to further consider three issues. The first is how to select the search domain $S$ in the low-dimensional embedding space. The second is how to specify a mapping $h$ associate with $B$ to map the point $\bm{z}$ from the embedding space $\mathcal{Z}$ to the high-dimensional space $\mathcal{D}$ where the response value $y$ should be evaluated. The third is how to maintain consistency of the Gaussian process regression model after updating $B$. Only by addressing these issues, the low-dimensional Bayesian optimization can be updated by the data pair $(\bm{z},y)$.

Before elaborating the selection of $S$ in the first issue, we introduce the definition of zonotope, which is a convex symmetric polytope.
\begin{definition}{Zonotope \citep{Le2013Zonotope}}
Given a linear mapping $H \in \mathbb{R}^{r \times d}$ and vector $\bm{p} \in \mathbb{R}^r$, an r-zonotope is defined as
\[
    Z = \{\bm{z}\in \mathbb{R}^r : \bm{z}= \bm{p}+H\bm{x}; \bm{x}\in [-1,1]^d \}
\]
where $\bm{p}$ is the center vector of the zonotope.
\end{definition}
Without loss of generality, let $\bm{p}=\bm{0}$, we note that $H$ is exactly $B$ computed by Equation \ref{matrixb}. Thus, the zonotope $Z_B$ associate with $B$ is the low-dimensional space where Bayesian optimization is performed. Then, we introduce the smallest box containing the zonotope.
\begin{lemma}
Given zonotope $Z = \{\bm{z}\in \mathbb{R}^r : \bm{z}= \bm{p}+H\bm{x}; \bm{x}\in [-1,1]^d \}$, $\bm{p}=\bm{0}$, the smallest box containing this zonotope can be computed as
\[
    S = \left[-\sum^{d}_{j=1}\left |  H_{1j} \right |,\sum^{d}_{j=1}\left |  H_{1j} \right |\right] \times \cdots \times \left[-\sum^{d}_{j=1}\left |  H_{rj} \right |,\sum^{d}_{j=1}\left |  H_{rj} \right |\right]
\]
\end{lemma}
Although the boundary of zonotope is difficult to solve, we use the box $S$ as an alternative \citep{binois2017choice}. 

Next, we focus on the second and third issues. As mentioned before, a mapping $h$ should connect the high-dimensional space and the low-dimensional space. As a result, for any point in the low-dimensional space, a corresponding point can be found in the high-dimensional space. In this way, the low-dimensional point $\bm{z}$ and the response value $y$ of the corresponding high-dimensional point $\bm{x}$ can be used as the training set to build a Gaussian process regression model. Figure \ref{fig:sketch} shows the relationship between them.
We note that the reproduction of the correlation between $\bm{z}$ and $y$ is the goal of the low-dimensional Gaussian process regression model, and the two are closely connected through $\bm{x}$. Thus a reasonable choice of $h$ plays a great influence on such a middle layer. Meanwhile, due to the iterative nature of SILBO, the mapping $h$ will change after updating $B$, which directly makes the correlation between $\bm{z}$ and $y$ inconsistent before and after the update. Therefore, we need to develop a strategy to maintain consistency. Next, we introduce two different strategies to address these two issues.

% \begin{figure}[H]
%     \centering
%     \begin{subfigure}
%         \centering
%         \includegraphics[width=3in,height =2.6in]{images/bu.png}
%         \caption{Bottom-up strategy}
%         \label{fig:hierarchy1}
%     \end{subfigure}
%     ~
%     \begin{subfigure}
%         \centering
%         \includegraphics[width=3in,height =2.6in]{images/td.png}
%         \caption{Top-down strategy}
%         \label{fig:hierarchy2}
%     \end{subfigure}
%     \caption{The relationship between response value $y$, high-dimensional point $x$, and low-dimensional point $z$. The solid line shows the process of mapping and evaluation with different strategies. The left hierarchy denotes the $bottom$-$up$ strategy. The dash line shows the process of re-mapping and re-evaluation. $z$ at the bottom is fixed after updating $B$. Then, re-map $z$ to find and re-evaluate the corresponding $x$. The right hierarchy shows the $top$-$down$ strategy. The dash line shows the re-mapping process. $x$ and $y$ at the top are fixed after updating $B$. Then, $x$ is re-mapped to find the corresponding $z$ and a new Gaussian Process regression model will be constructed.}
%     \label{fig:sketch}
% \end{figure}

\begin{algorithm}
	\caption{The $bottom$-$up$ strategy}\label{algorithm3}
	\KwIn{Labeled low-dimensional points $Z_l^{t}$, unlabeled low-dimensional points $Z_u^{t}$}
	\KwOut{Labeled high-dimensional points $X_l^{t}$, unlabeled high-dimensional points $X_u^{t}$, the training set $D_{t}$}
	
	$X_l^{t} = B_{t}^{\dagger}Z_l^{t}$\;
	$X_u^{t} = B_t^{\dagger}Z_u^{t}$\;
	$Y_{t} = f(X^{t}_l)$, $D_{t} = (Z_l^{t},Y_{t}) $\;
	
\end{algorithm}

As shown in Algorithm \ref{algorithm3}, we can naturally get the response value $y$ by evaluating the point $\bm{x}=B^{\dagger}\bm{z}$ after finding the candidate point $\bm{z} \in S$ through BO. When the projection matrix $B$ changed, we fix $z$ at the bottom of the hierarchy in Figure \ref{fig:hierarchy1} and then update $y$ at the top to maintain consistency~($bottom$-$up$ for short).

\begin{figure}[t]
	\centering
	\begin{subfigure}{0.4\textwidth} % width of left subfigure
		\includegraphics[width=3in,height =2in]{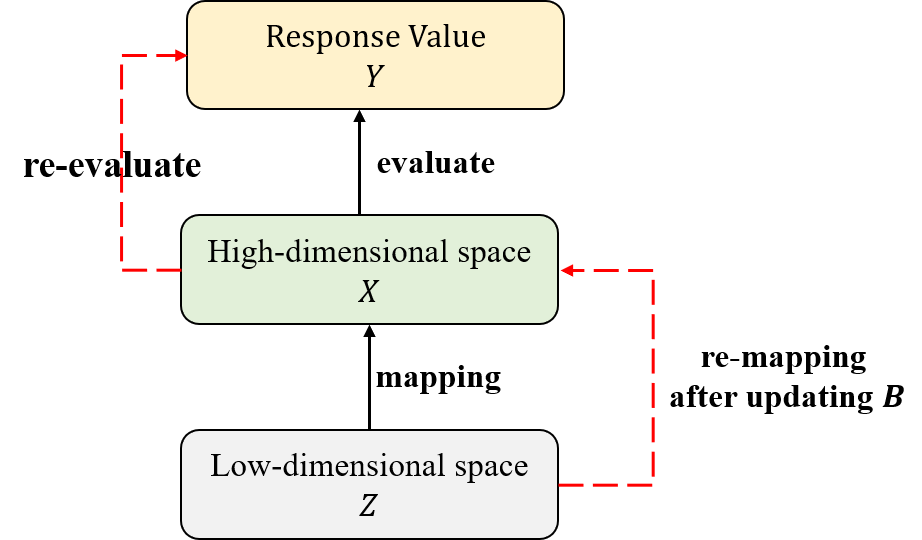}
		\caption{Bottom-up strategy} % subcaption
		\label{fig:hierarchy1}
	\end{subfigure}
	\hspace{7em} % here you can insert horizontal or vertical space
	\begin{subfigure}{0.4\textwidth} % width of right subfigure
		\includegraphics[width=2.5in,height =2in]{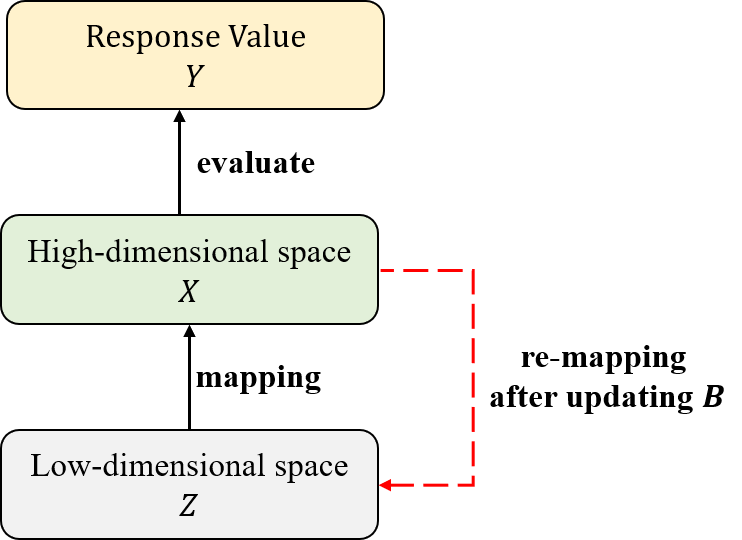}
		\caption{Top-down strategy} % subcaption
		\label{fig:hierarchy2}
	\end{subfigure}
	\caption{The relationship between response value $y$, high-dimensional point $x$, and low-dimensional point $z$. The solid line shows the process of mapping and evaluation with different strategies. The left hierarchy denotes the $bottom$-$up$ strategy. The dash line shows the process of re-mapping and re-evaluation. $z$ at the bottom is fixed after updating $B$. Then, re-map $z$ to find and re-evaluate the corresponding $x$. The right hierarchy shows the $top$-$down$ strategy. The dash line shows the re-mapping process. $x$ and $y$ at the top are fixed after updating $B$. Then, $x$ is re-mapped to find the corresponding $z$ and a new Gaussian Process regression model will be constructed.} % caption for whole figure
	\label{fig:sketch}
	\vspace{-1ex}
\end{figure}

%Compared with REMBO \citep{Wang2013Bayesian}, the $BU$ strategy greatly reduces the size of low-dimensional search space.

Since we use $B^{\dagger}$ to connect the points between the low-dimensional space $\mathcal{Z}$ and the high-dimensional space $\mathcal{D}$, the actual evaluation space $\phi$ is only a part of $\mathcal{D}$~(i.e., $\phi \subset \mathcal{D}$) when $B$ is fixed. While a great probability of containing the optimum in $\phi$ has been proved \citep{Wang2013Bayesian}, one of the preconditions is that there is no limitation to the high-dimensional space of the objective function. In practice, the restriction often exists~(e.g., $[-1,1]^d$), making the high probability no longer guaranteed. In contrast, due to the iterative nature of the effective low-dimensional space learning mechanism in SILBO~(Algorithm \ref{algorithm1}), we can alleviate this problem. Specifically, more and more $\phi$ associate with $B$ will be learned, which gradually reveal the location of the optimum.

However, the $bottom$-$up$ strategy may bring more evaluation overhead. When we get the training set $D_{t}=\{\bm{z}_i,y_i \}_{i=1}^{T}$ for BO through $T$ iterations in Algorithm \ref{algorithm1}, we update $B$ to $B_{new}$ and get $\bm{z}_{T+1}$ through the acquisition function $\alpha$. Due to the updating of $B$, the mapping relationship between the high-dimensional and low-dimensional spaces has changed. As shown in Figure \ref{fig:hierarchy1}, we need to find the corresponding $\bm{x}$ again for $\bm{z}$ in the training set and then evaluate $\bm{x}$ to maintain consistency. When the evaluation of the objective function is expensive, the computational cost is huge.

To eliminate the re-evaluation overhead, we propose a new strategy. Specifically, we first obtain $\bm{z}$ in the low-dimensional space through the acquisition function $\alpha$, and then solve the following equation to find the corresponding $\bm{x}$.
\begin{equation}
   \mathop{\arg\min}\limits_{\bm{x}} ||B\bm{x}-\bm{z}||^2 \label{strategy2}
\end{equation}
When the projection matrix $B$ is changed, we fix $\bm{x}$ and $y$ at the top of the hierarchy in Figure \ref{fig:sketch} and then update $\bm{z}$ in the bottom to maintain consistency~($top$-$down$ for short). 

\begin{algorithm}[h]
\caption{The $top$-$down$ strategy}\label{algorithm4}
\KwIn{Labeled low-dimensional points $Z_l^{t}$, unlabeled low-dimensional points $Z_u^{t}$, labeled high-dimensional points $X_l^{t-1}$, the training set $D_{t-1}$}
\KwOut{labeled high-dimensional points $X_l^{t}$, unlabeled high-dimensional points $X_u^{t}$, the training set $D_{t}$}

Get the last point $z_l$ from $Z_l^{t}$\;
Compute $x_l$,$X_u^{t}$ from Equation \ref{strategy2} using $z_l$,$Z_u^t$\;
$X^t_l = X_l^{t-1} \cup x_l$\;
$y = f(x_l)$, $D_t = D_{t-1}\cup(\bm{z}_l,y)$\;

\end{algorithm}

Different from the $bottom$-$up$ strategy, the $top$-$down$ strategy shown in Figure \ref{fig:hierarchy2} updates $\bm{z}$ directly, which enables us to reconstruct a new BO model efficiently only by replacing the training set with $\{B_{new}\bm{x}_i,y_i \}_{i=1}^{T}$ after updating $B$, instead of relocating $\bm{x}$ and re-evaluating them. The $top$-$down$ strategy can be found in Algorithm \ref{algorithm4}. Next, we analyze the rationality of this strategy theoretically.

\begin{theorem}\label{theorem}
Given a matrix $B \in \mathbb{R}^{r\times d}$ with orthogonal rows and its corresponding zonotope $Z_B$. Given $\bm{z} \in Z_B$, for any $\bm{x}_1,\bm{x}_2 \in U_B=\{\bm{x}|\bm{x}=B^{\top}\bm{z}+y,\bm{x} \in \mathbb{R}^d, y\in \text{N}(B)  \}$, $B\bm{x}_1 = B\bm{x}_2 = \bm{z}$. 
\end{theorem}
\begin{proof}
Let $\bm{x} = \bm{x}_V+\bm{x}_V^{\bot} \in \mathbb{R}^d$ be the orthogonal decomposition with respect to $V$ = Row($B$), then $\bm{x}_V \in V$. Let $\bm{z} \in Z_B$ with $B^{\top}\bm{z} = \bm{x}_V$. Due to the orthogonal decompostion, we have $\bm{x}-\bm{x}_V = \bm{x} - B^{\top}\bm{z}$ lies in $V^{\bot}=\text{N}(B)$. For any $\bm{x}_1, \bm{x}_2\in U_B$, we have $\bm{x}_1-\bm{x}_V = \bm{x}_1 - B^{\top}\bm{z}$ and $\bm{x}_2-\bm{x}_V = \bm{x}_2 - B^{\top}\bm{z}$, then $B( \bm{x}_1 - B^{\top}\bm{z})=B( \bm{x}_2 - B^{\top}\bm{z})=0$. Thus, $B\bm{x}_1 = B\bm{x}_2 = BB^{\top}\bm{z} = \bm{z}$.
\end{proof}

According to Theorem \ref{theorem}, if we assume that $B$ has orthogonal rows, it is clear that given $B$ and $\bm{z}$, any point $\bm{x} \in U_B$ will be the solution to Equation \ref{strategy2}. However, our purpose is not $\bm{x}$ itself, but $y=f(\bm{x})$, because $(\bm{z},y)$ is the data pair used to update BO. Fortunately, if we use SILBO to learn an ideal effective low-dimensional subspace $B^*$, then for any $\bm{x}_1,\bm{x}_2$ in the solution set $U_{B^*}$, $f(\bm{x}_1)=f(B^{*\top}\bm{z}+y_1) = f(B^{*\top}\bm{z})$, $f(\bm{x}_2)=f(B^{*\top}\bm{z}+y_2) = f(B^{*\top}\bm{z})$~(i.e., $f(\bm{x}_1)=f(\bm{x}_2)$). Thus, in each iteration $T$, we can obtain the unique data pair $ (\bm{z}_T,y_T)$ under $B^*$ to update the low-dimensional Bayesian optimization model. Therefore, the diversity of solutions in the set $U_B$ does not affect the construction of BO and we can explore in the original space $\mathcal{D}$ which liberates the shackles of $\phi$.

In summary, both the two strategies aim to generate a consistent training set for the subsequent GP construction when updating $B$. To maintain such consistency, the $bottom$-$up$ strategy acquires more observations $y$ from the objective function according to $\bm{z}$ in the training set while the $top$-$down$ strategy directly changes $\bm{z}$ according to $y$. Note that the more response values, the more information about the objective function we can get. Thus, the $bottom$-$up$ strategy can obtain more clues about the global optimum than the $top$-$down$ strategy.

% numerial result
\section{Numerical Results}\label{numericalresults}
In this section, we evaluated our proposed algorithm SILBO on a set of high-dimensional benchmarks. First, we compared SILBO with other state-of-the-art algorithms on synthetic function\footnote{The synthetic function can be found at \url{https://www.sfu.ca/~ssurjano/optimization.html.}} and neural network hyperparameter optimization tasks. Second, we demonstrated the effectiveness of semi-supervised embedding learning and iterative learning. Finally, we evaluated and analyzed the scalability of SILBO.

\subsection{Experiment Setting}
For comparison, we also evaluated other state-of-the-art algorithms: REMBO \citep{Wang2013Bayesian}, which performs BO on a random generated linear embedding space, REMBO-$K_X$, which computes the kernel using the distance in the high-dimensional space, REMBO-$K_{\psi}$ \citep{binois2014warped}, which uses a warping kernel and finds the evaluation point through sophisticated mapping, and HeSBO \citep{nayebi19a}, which uses the count-sketch technique to generate subspace and compute the projection matrix. Additional methods include ADD-BO \citep{k2015high}, which assumes that the objective function has an additive structure, and recently-proposed SIR-BO \citep{zhang2019high}, which uses SIR to learn a low-dimensional space purely. $\text{SILBO-BU}$ and $\text{SILBO-TD}$ represent the proposed SILBO algorithm that employs the $bottom$-$up$ and $top$-$down$ mapping strategies respectively.

We employed the package \textit{GPy}\footnote{\url{https://github.com/SheffieldML/GPy}} as the Gaussian Process framework and selected the Matérn kernel for the GP model. We adopted the package ristretto\footnote{\url{https://github.com/erichson/ristretto}} to perform the R-SVD. 
We employed \textit{CMA-ES} \citep{hansen16CMA} to compute Equation \ref{strategy2} for $\text{SILBO-TD}$.
The embedding matrix $B$ is updated every 20 iterations. The size of the unlabeled points is set to 50 in each iteration. The number of neighbors $k$ is set to $7$ and the number of unlabeled points is set to $50$ for semi-supervised dimension reduction.
To obtain error bars, we performed 100 independent runs for each algorithm on synthetic functions and 5 runs for the neural network hyperparameter search. We plotted the mean value along with %median$\pm$ 
one standard deviation. For ADD-BO, we used the authors' implementation through $MATLAB$, so we did not compare its scalability since other algorithms were implemented in $Python$. 
Also, we evaluated SILBO using two acquisition functions: UCB and EI.
The number of initial random points defaults to 50 and the response values of initial points are not considered when observing the experimental performance except in Section~\ref{effectiveofiterativelearning}, which uses different number of initial points, so the response values of these points are shown in the performance curve.

\subsection{Performance on Synthetic Functions}

%branin
\begin{figure}
    \centering
    \begin{subfigure}[t]{0.4\textwidth}
        \centering
        \includegraphics[width=1\textwidth]{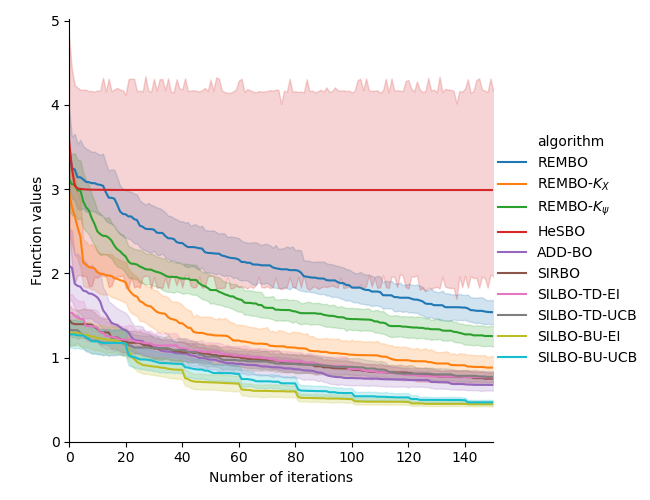}
        \caption{100d Branin}
    \end{subfigure}%
    ~ 
    \begin{subfigure}[t]{0.4\textwidth}
        \centering
        \includegraphics[width=1\textwidth]{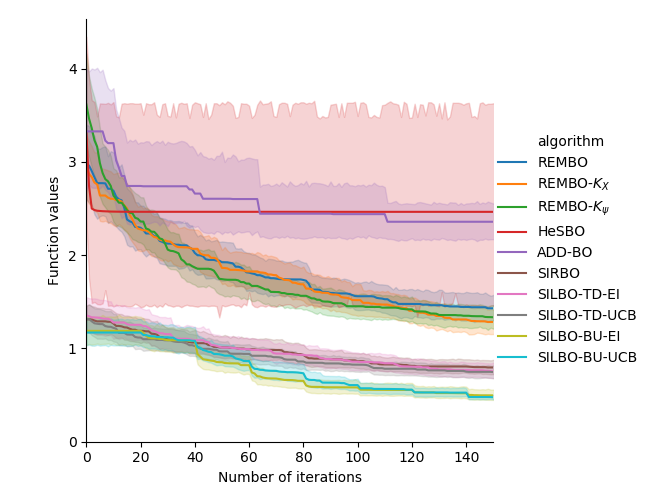}
        \caption{1000d Branin}
    \end{subfigure}
    \caption{Performance on Branin under dimension 100(a) and dimension 1000(b) with embedding dimension $r$ = 2. We plotted mean and one standard deviation across 100 independent runs.}
    \label{fig:branin}
\end{figure}

%figure hartmann6

\begin{figure}
    \centering
    \begin{subfigure}[t]{0.4\textwidth}
        \centering
        \includegraphics[width=1\textwidth]{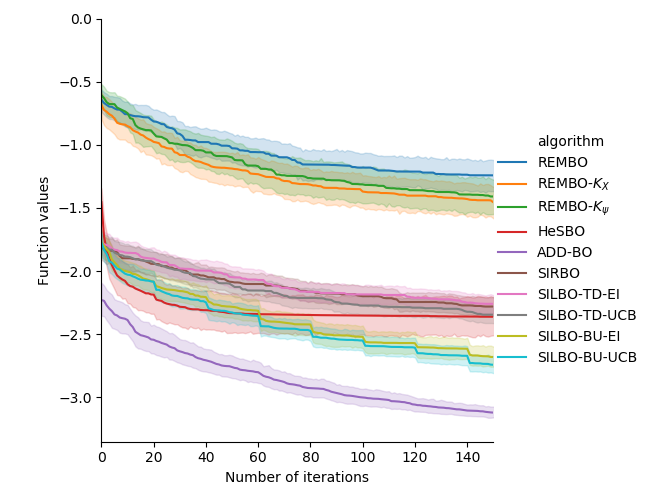}
        \caption{100d hartmann6}
    \end{subfigure}%
    ~ 
    \begin{subfigure}[t]{0.4\textwidth}
        \centering
        \includegraphics[width=1\textwidth]{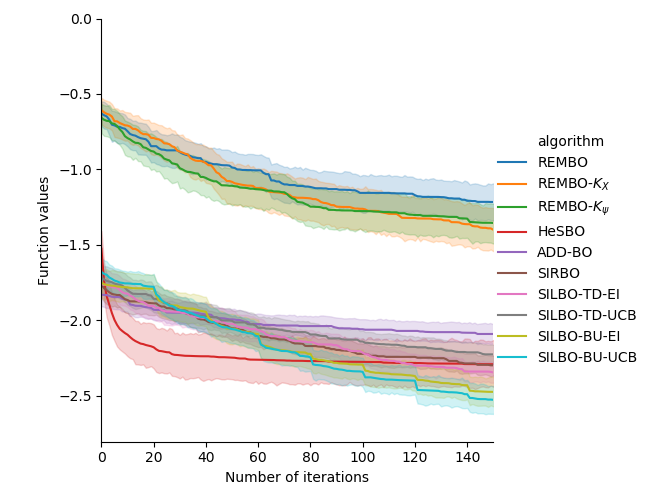}
        \caption{1000d hartmann6}
        \label{fig:hartmann-1000}
    \end{subfigure}
    \caption{Performance on Hartmann6 under dimension 100(a) and dimension 1000(b) with embedding dimension $r$ = 6. We plotted mean  and one standard deviation across 100 independent runs.}
    \label{fig:hartmann}
\end{figure}

%figure colville

\begin{figure}[htb]
    \centering
    \begin{subfigure}[t]{0.4\textwidth}
        \centering
        \includegraphics[width=1\textwidth]{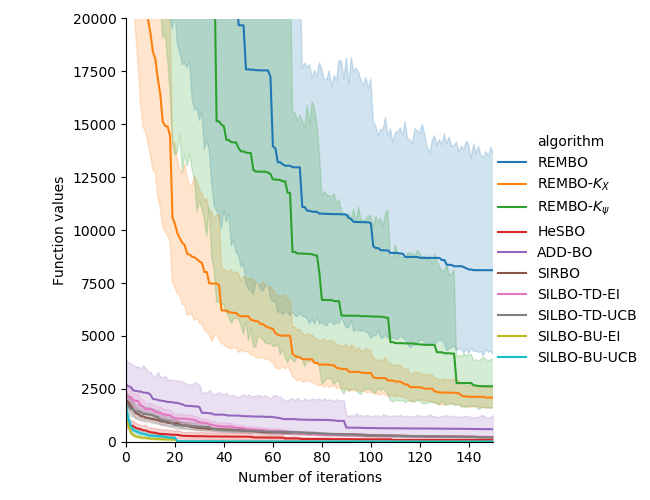}
        \caption{100d Colville}
    \end{subfigure}\hfil
    \begin{subfigure}[t]{0.4\textwidth}
        \centering
        \includegraphics[width=1\textwidth]{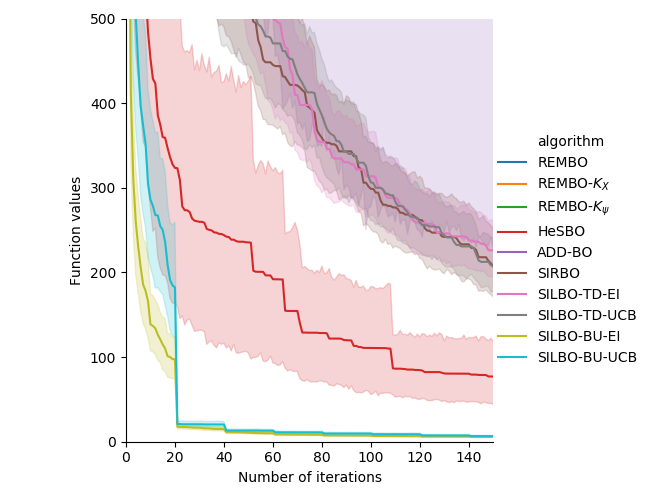}
        \caption{100d Colville~(enlarged)}
    \end{subfigure}
    \medskip
    \begin{subfigure}[t]{0.4\textwidth}
        \centering
        \includegraphics[width=1\textwidth]{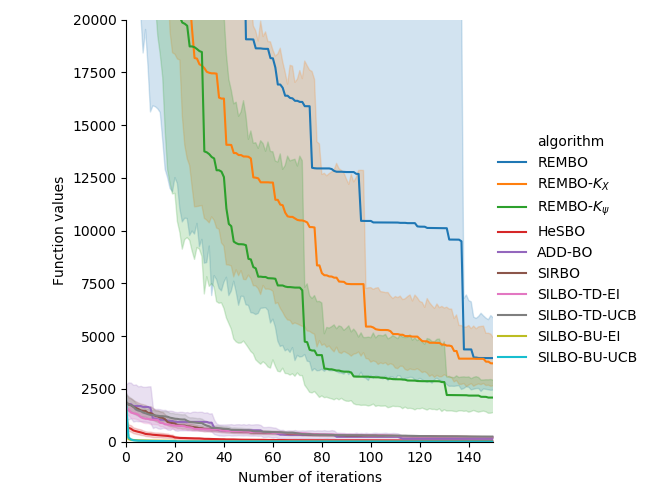}
        \caption{1000d Colville}
    \end{subfigure}\hfil
    \begin{subfigure}[t]{0.4\textwidth}
        \centering
        \includegraphics[width=1\textwidth]{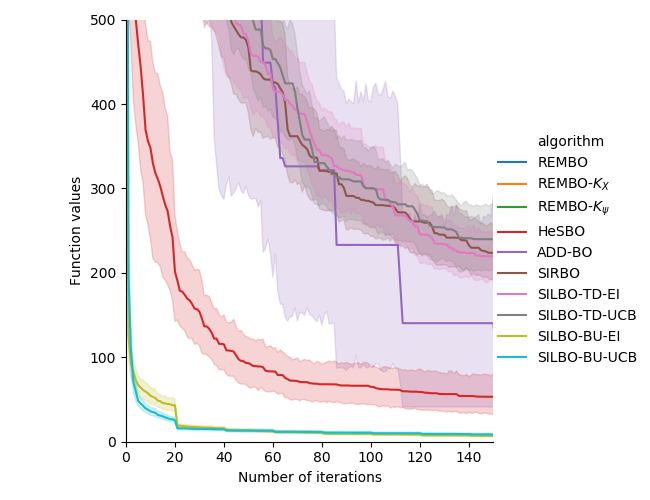}
        \caption{1000d Colville~(enlarged)}
    \end{subfigure}
    \caption{Performance on Colville under dimension 100(a)(b) and dimension 1000(c)(d) with embedding dimension $r$ = 4. To see performance discrepancy clearly among different algorithms, we enlarged~(a)(c) to get~(b)(d) respectively. We plotted mean and one standard deviation across 100 independent runs.}
    \label{fig:colville}
\end{figure}

Similar to \citep{nayebi19a}, these algorithms were under comparison upon the following widely-used test functions:~(1) Branin~(2) Hartmann-6~(3) Colville. The active dimensions for Branin,  Hartmann-6, Colville are 2, 6, and 4 respectively. We studied the cases with different input dimensions $d \in \{100,1000\}$. The experimental results are shown in Figure \ref{fig:branin}-\ref{fig:colville}. The proposed $\text{SILBO-BU}$, outperforms all benchmarks especially in high dimensions~($d=1000$). $\text{SILBO-TD}$ is also very competitive, surpassing most traditional algorithms. In addition, the experimental results under different acquisition functions are similar. ADD-BO performs excellently in Hartmann-6 due to the additive structure of the objective function itself, but it performs poorly in the 1000-dimension case (Figure~\ref{fig:hartmann-1000}) and in functions without the additive structure. HeSBO performs well in Colville and Hartmann-6, but poorly in Branin. Moreover, we can find that HeSBO does not converge in most cases. The performance of SIRBO is similar to our proposed $\text{SILBO-TD}$ since it also tries to learn an embedding space actively. In summary, the proposed method $\text{SILBO-BU}$ nearly beats all baselines due to effective semi-supervised dimension reduction and iterative embedding learning.

%%%%%%%%%%%%%%%%%%%%%%%%%%%%%%%
\subsection{Hyperparameter Optimization on Neural Network}
Following \citep{oh2018bock}, we also evaluated SILBO in a hyperparameter optimization task for a neural network on the \textit{MNIST} dataset. Specifically, the neural network is a multi-Layer perceptron~(MLP) with one hidden layer of size 50 and one output layer of size 10. These 500 initial weights were viewed as hyperparameters and were optimized using the Adam \citep{kingma2014adam} optimizer. 

\begin{figure}[htb]
    \centering
    \begin{subfigure}[t]{0.5\textwidth}
        \centering
        \includegraphics[width=0.9\textwidth]{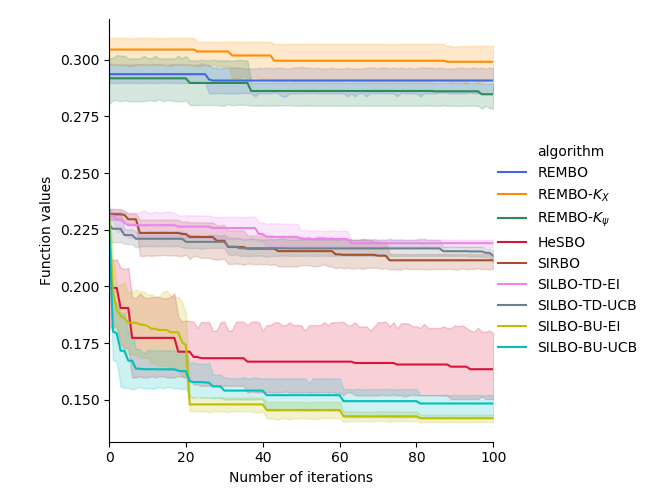}
    \end{subfigure}
    \caption{Performances on the MLP with embedding dimension $r$= 12. We plotted mean and one standard deviation across 5 independent runs.}
    \label{fig:neural}
\end{figure}

Figure \ref{fig:neural} shows that $\text{SILBO-BU}$ converges faster than all other algorithms. The performance of $\text{SILBO-TD}$ and SIRBO is close. These two methods perform better than other traditional methods such as REMBO and show competitive performance. Similar to the performance on synthetic functions, $\text{SILBO-BU}$ beats all other high-dimensional BO methods and can obtain better response values in the neural network hyperparameter optimization tasks.
Next, we evaluated the effectiveness of semi-supervised dimension reduction and iterative embedding learning respectively.

%%%%%%%%%%%%%%%%%%%%%%%%%%%%%%%
\subsection{Effectiveness of Semi-supervised Dimension Reduction}
For the function with effective dimensions, the low-dimensional subspace can reflect its intrinsic structure information. Thus, if we sample points from the function and project them into the effective subspace, the response values of these points will be dispersed along the direction of each subspace basis. Here, we compared the embedding learning performance of SILBO and SIR on Branin and Camel functions. We first evaluated 50 high-dimensional (i.e., 1000 dimensions) points to get their response values and then projected these points into the low-dimensional space generated by SILBO and SIR respectively. The goal of embedding learning is to find the e.d.r directions that preserve the information of the objective function as much as possible. Figure \ref{fig:effectivecamel} and Figure \ref{fig:effectivebranin} summarize the observed performance. 
\begin{figure}[htb]
\centering
    \begin{subfigure}{0.4\textwidth}
    \centering
        \includegraphics[width=2in,height =1.4in]{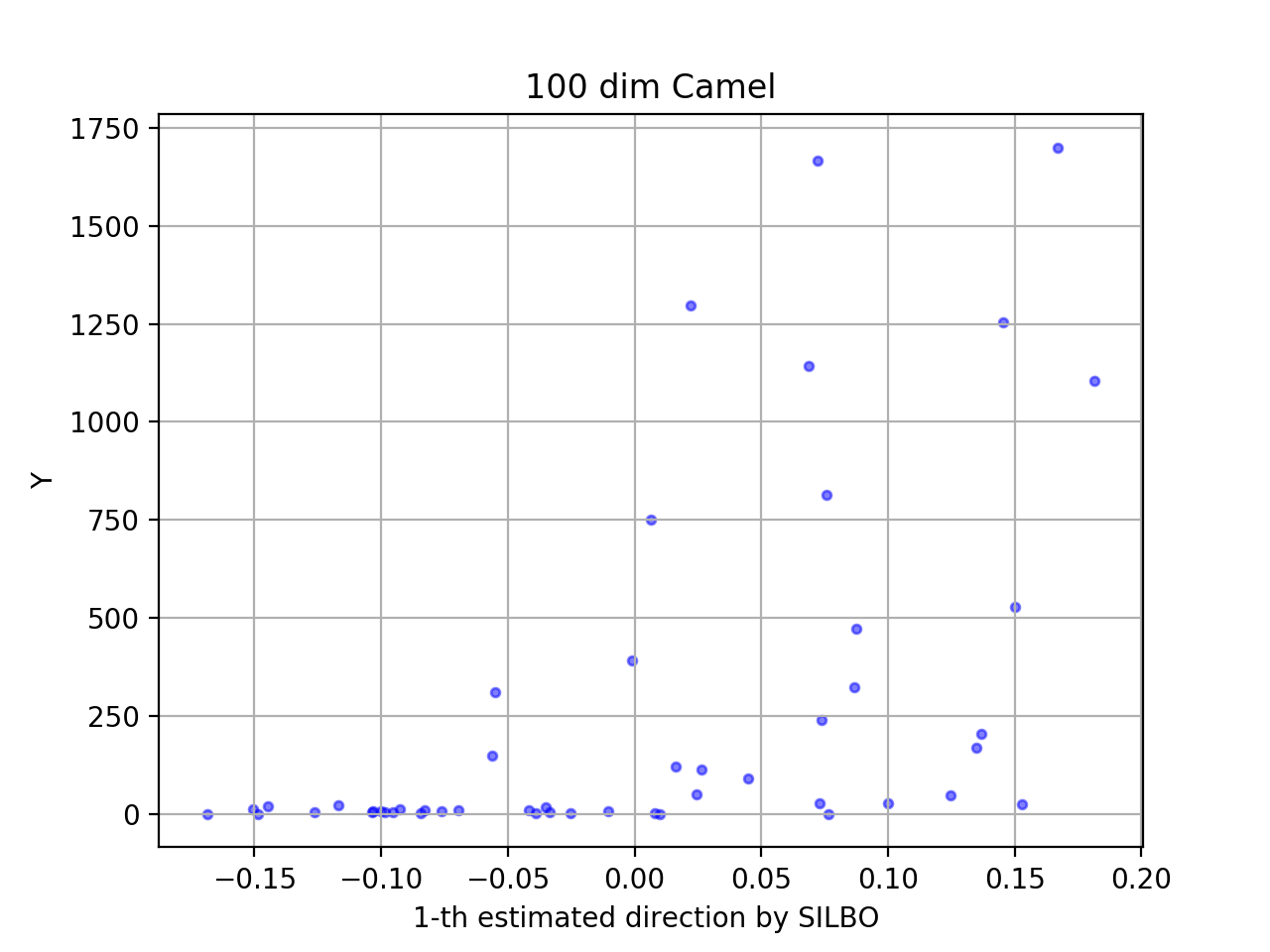}
        \caption{}
    \end{subfigure}\hfil
    \begin{subfigure}{0.4\textwidth}
    \centering
        \includegraphics[width=2in,height =1.4in]{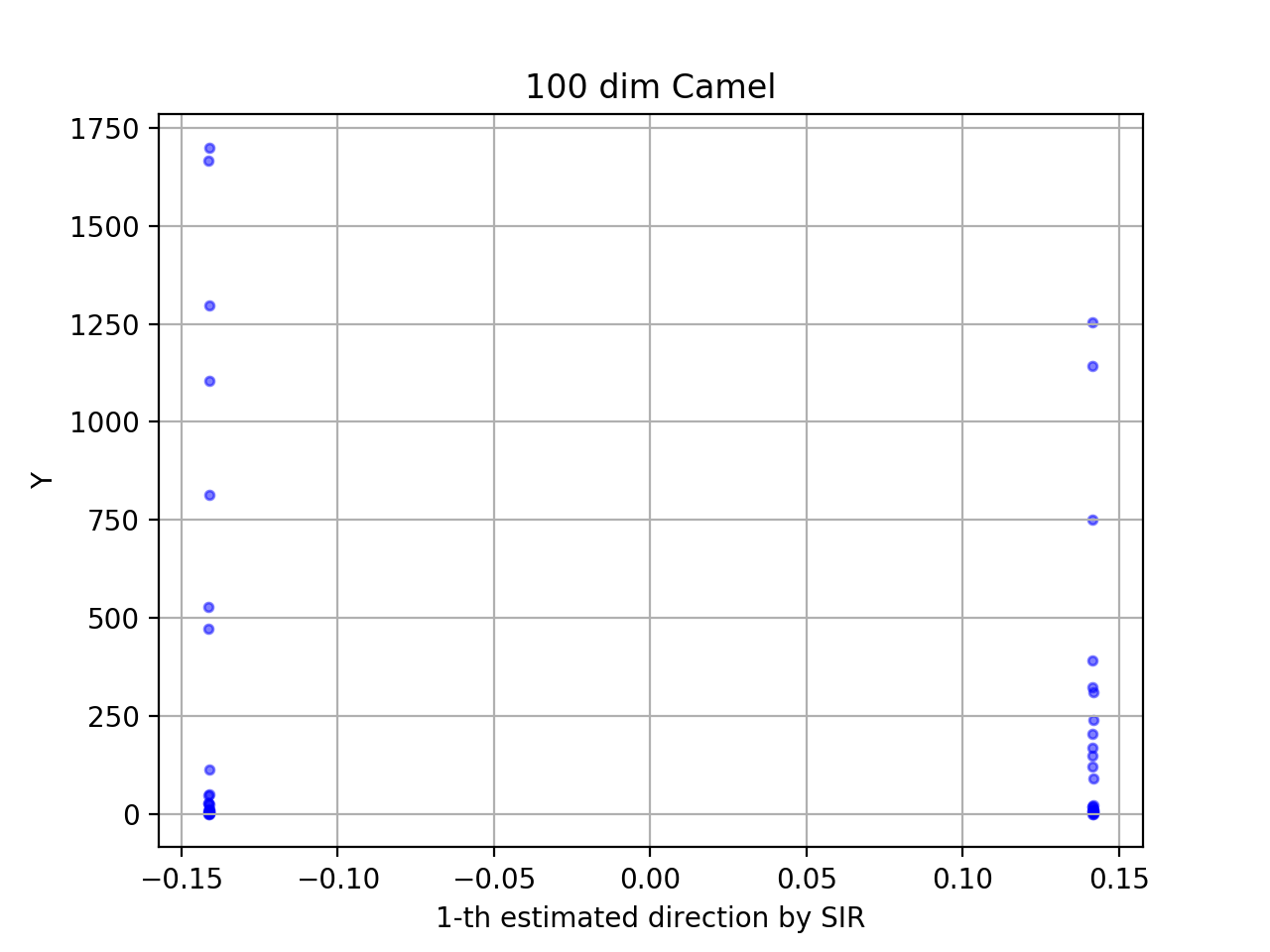}
        \caption{}
         \label{fig:effectivecamel-sir}
    \end{subfigure}
    \caption{Point distribution in the low-dimensional space generated by SILBO and SIR respectively on Camel. The number of e.d.r directions is 1.}
    \label{fig:effectivecamel}
\end{figure}

\begin{figure}[htb]
    \centering
    \begin{subfigure}{0.4\textwidth}
    \centering
        \includegraphics[width=2in,height =1.4in]{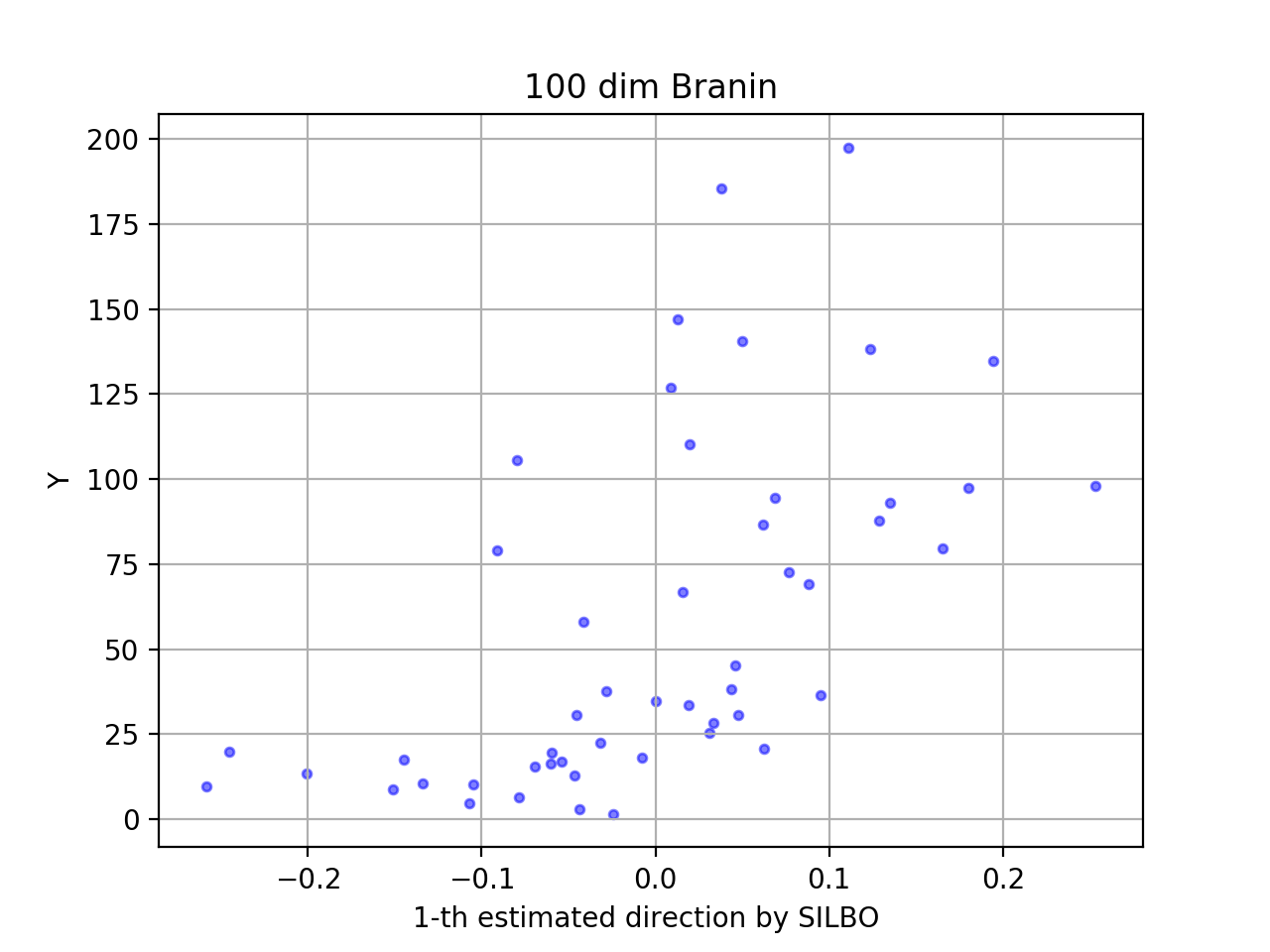}
        \caption{}
    \end{subfigure}\hfil
    \begin{subfigure}{0.4\textwidth}
    \centering
    \includegraphics[width=2in,height =1.4in]{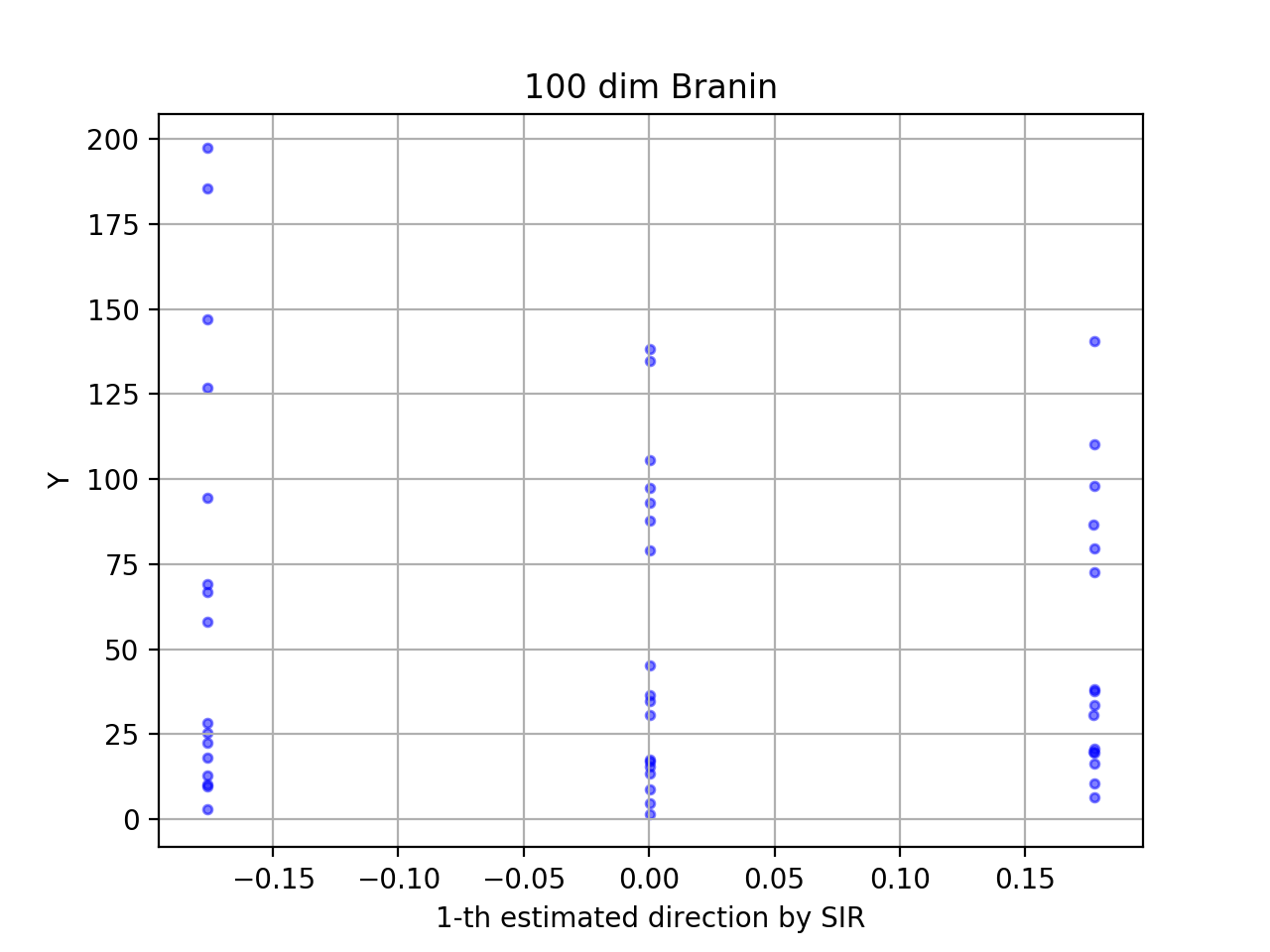}
        \caption{}
           \label{fig:effectivebranin-sir-2-1}
    \end{subfigure}\hfil
    \medskip
    \begin{subfigure}{0.4\textwidth}
    \centering
        \includegraphics[width=2in,height =1.4in]{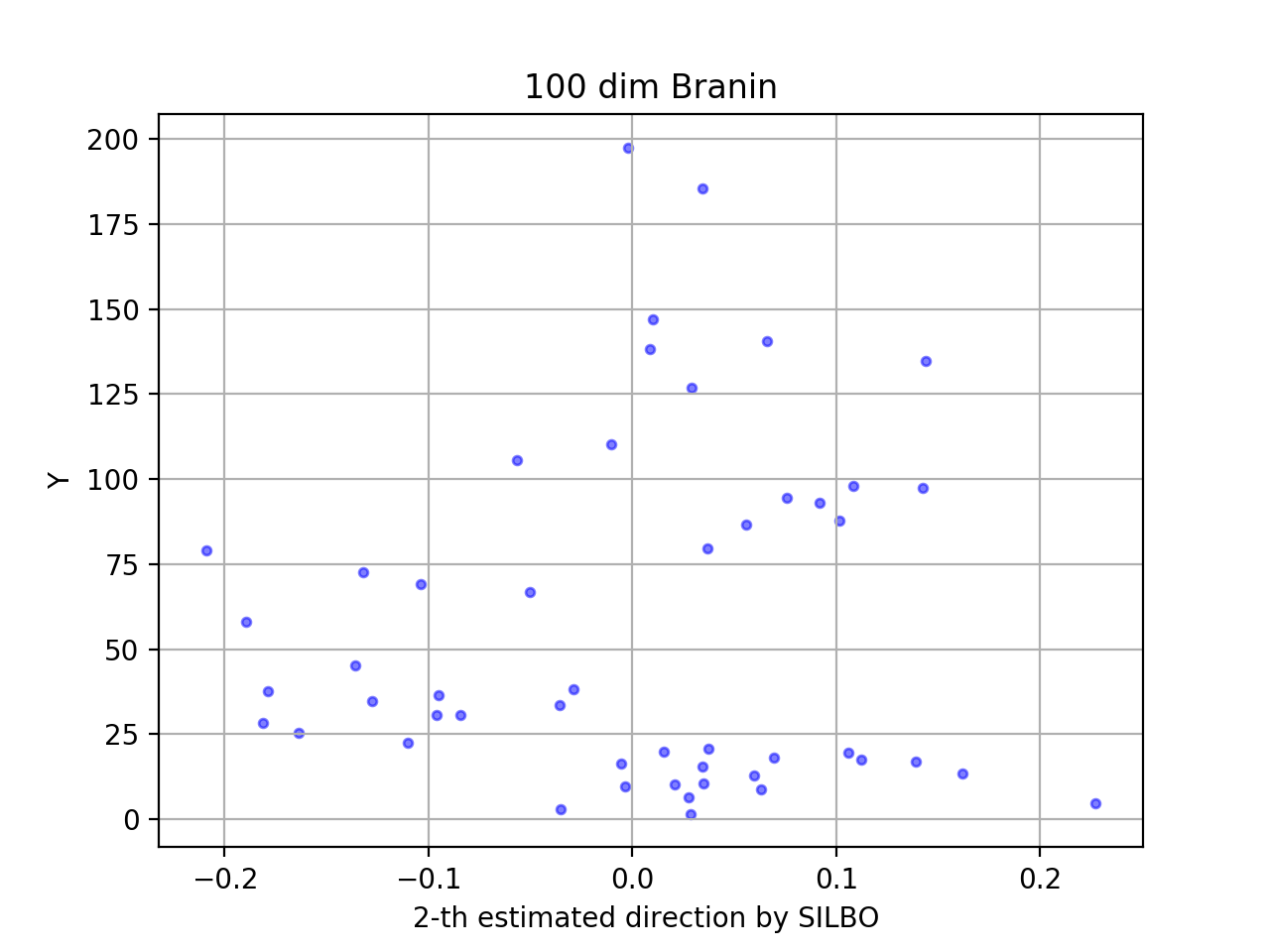}
        \caption{}
    \end{subfigure}\hfil
    \begin{subfigure}{0.4\textwidth}
    \centering
        \includegraphics[width=2in,height =1.4in]{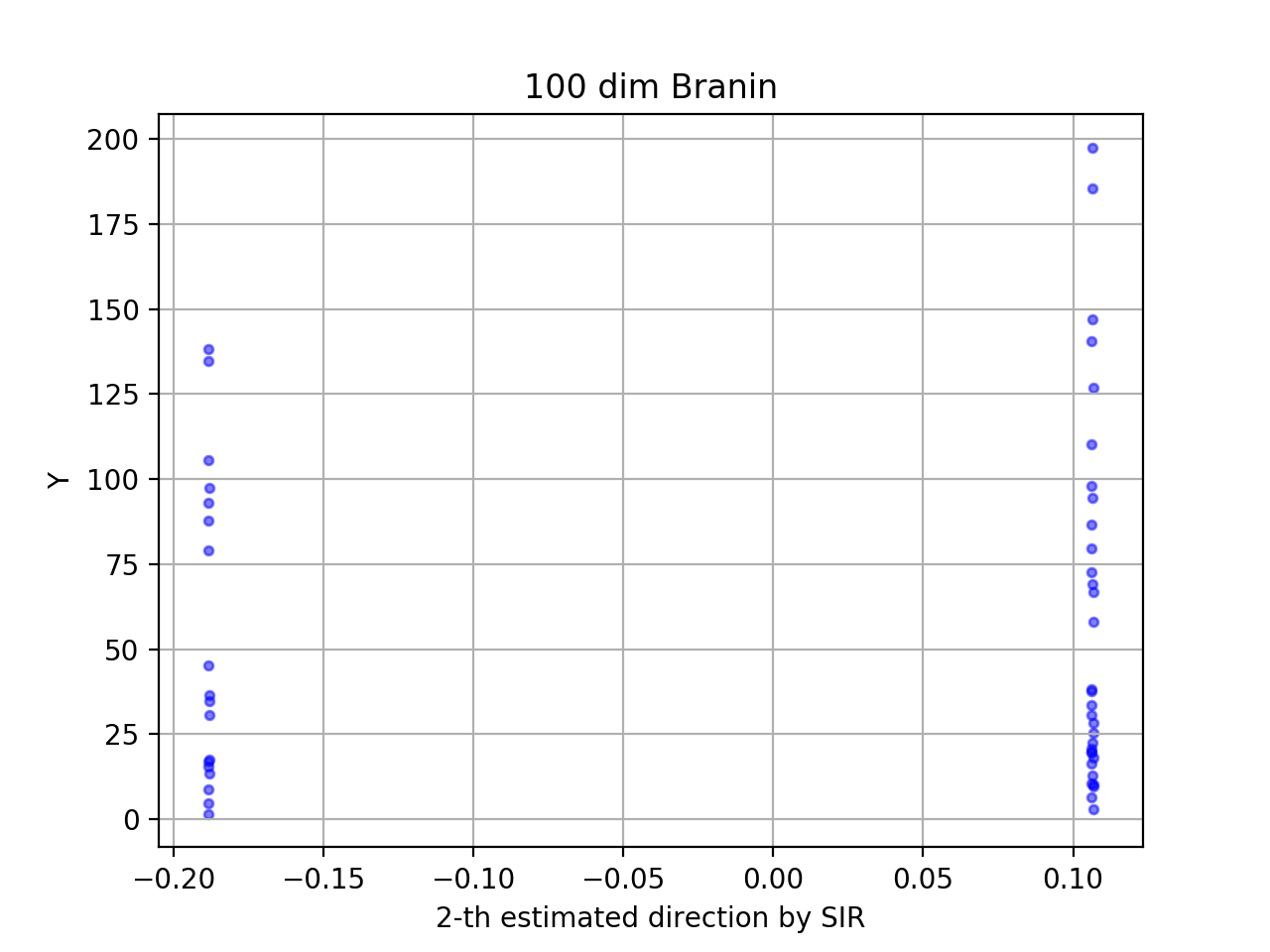}
        \caption{}
        \label{fig:effectivebranin-sir-2-2}
    \end{subfigure}\hfil
    \caption{Point distribution in the low-dimensional space generated by SILBO and SIR respectively on Branin. The number of e.d.r directions is 2.~(a)(b) plot the points projected to one estimated direction.~(c)(d) plot the points projected to the other direction.}
    \label{fig:effectivebranin}
\end{figure}

In Figure~\ref{fig:effectivecamel-sir}, Figure~\ref{fig:effectivebranin-sir-2-1}, and Figure~\ref{fig:effectivebranin-sir-2-2}, the information extracted by SIR stacks together. A large number of points are distributed in a very small interval, but correspond to many different response values with large variance.
The consequence is that a lot of information will be lost if we train a Gaussian process regression model using these low-dimensional representations. In contrast, due to the semi-supervised dimension reduction, the information extracted by SILBO is more complete without losing the intrinsic information of the objective function. 
%
%Here we note that the number of initial points also has an impact on the performance of embedding learning. If we use more points, such as 100 initial points, the stack phenomenon will disappear when using SIR. 

To further demonstrate the effectiveness of semi-supervised dimension reduction, we also evaluated SILBO only using the labeled points in each iteration. We call this method $\text{SILBO-only-labeled}$.  Figure \ref{fig:labeledvsunlabeled} shows the comparison results between $\text{SILBO-semi}$ and $\text{SILBO-only-labeled}$ with the $bottom$-$up$ mapping strategy. From Figure \ref{fig:labeledvsunlabeled}, we can see that the semi-supervised dimension reduction technique can significantly improve the performance of high-dimensional Bayesian optimization.

% while the $top$-$down$ strategy is insensitive to the embedding space. We think that this is due to the size of the sampling space. $top$-$down$ strategy samples points from the high-dimensional space but construct the GP model in the low-dimensional space. Though updated iteratively, a better embedding is still difficult to make up for the gap between the number of points selected by acquisition function (one point per iteration) and the dimensionality of high-dimensional space. That is, for $top$-$down$ strategy, if we increase the number of iterations or increase the number of initial points, the choice of embedding space will play a more important role. $\text{SILBO-BU}$ samples points and construct the GP model directly from the low-dimensional space, the update of projection matrix will affect the response value of these points, and then strongly affect the construction of low dimensional GP model. Finally, $\text{SILBO-BU-UCB}$ significantly outperforms other methods on different objective functions, which indicates that the semi-supervised technique is effective for high-dimensional Bayesian optimization.

\begin{figure}[t]
    \centering
    \begin{subfigure}[H]{0.3\textwidth}
    \centering
    \includegraphics[width=1.9in,height=1.2in]{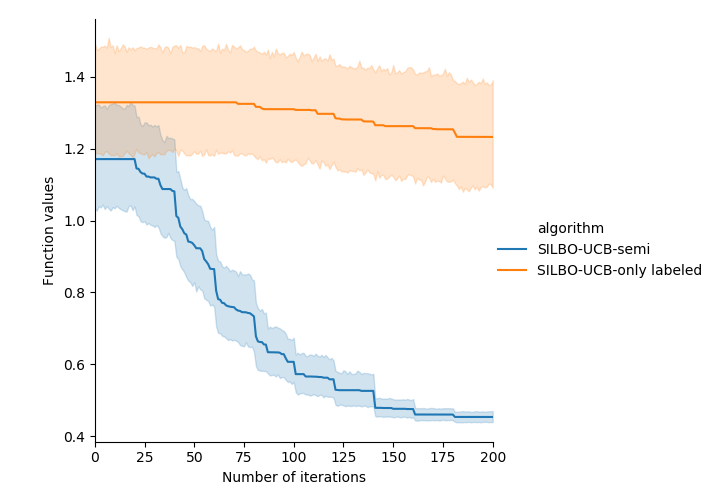}
        \caption{1000d Branin}
    \end{subfigure}
    ~
    \begin{subfigure}[H]{0.3\textwidth}
    \centering
        \includegraphics[width=1.9in,height=1.2in]{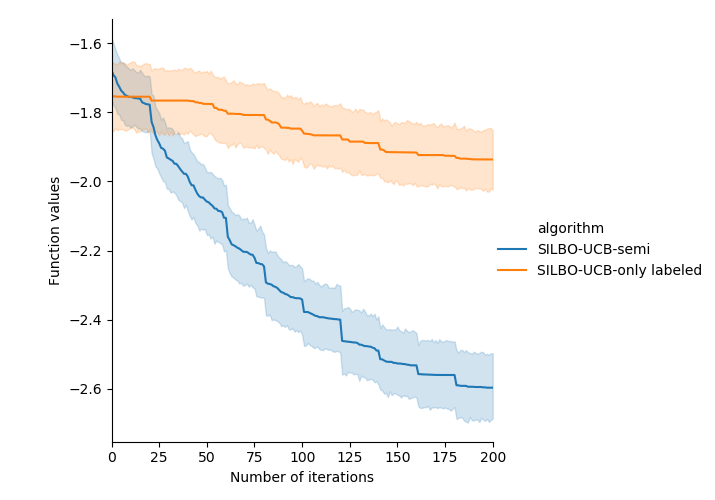}
        \caption{1000d Hartmann6}
    \end{subfigure}
    ~
    \begin{subfigure}[H]{0.3\textwidth}
    \centering
        \includegraphics[width=1.9in,height=1.2in]{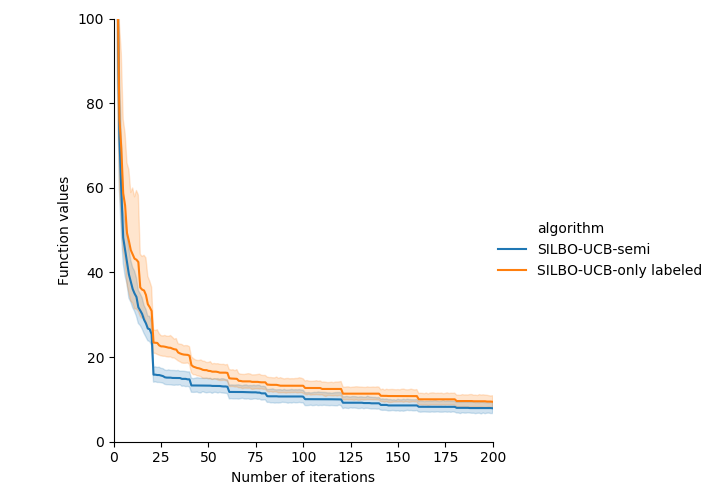}
        \caption{1000d Colville}
    \end{subfigure}
    \caption{Performance comparison between SILBO-semi and SILBO-only-labeled on three synthetic functions. All of the methods update the projection matrix every 20 iteration.}
    \label{fig:labeledvsunlabeled}
\end{figure}

\subsection{Effectiveness of Iterative Learning}\label{effectiveofiterativelearning}
In the iterative process of SILBO, the projection matrix is updated based on random initial points and the points acquired from the acquisition function. 
To verify the effectiveness of iterative learning, we need to show that the points acquired from the acquisition function indeed contribute to generating a better embedding.
Figure \ref{fig:iterativeeffectiveness} shows the performance comparison results with and without the iterative learning process. $\text{SILBO-only-random}$ means that the low-dimensional embedding is learned only using the random initial points, while $\text{SILBO-random+acquisition}$ contains the random initial points and the points from the acquisition function. 

We chose 20 random initial points for $\text{SILBO-only-random}$ and chose 10 random initial points accompanied with 10 points acquired from the acquisition function iteratively for $\text{SILBO-random+acquisition}$. 
We compared the optimization performance of the two algorithms after these 20 points. From Figure \ref{fig:iterativeeffectiveness}, we can see that $\text{SILBO-random+acquisition}$ can find a better solution at most cases, which indicates the effectiveness of iterative learning. 
The main reason why $\text{SILBO-random+acquisition}$ performs slightly worse on Branin is due to less initial points.
%
%As mentioned before, the number of initial points also matters a lot. 
Moreover, the performance of $\text{SILBO-random+acquisition}$ in Figure~  \ref{fig:iterativeeffectiveness} is worse than other experiments for two reasons. One is that the initial points are much less, we only use 20 initial points. The other is that the low-dimensional embedding space is only updated once.

Furthermore, we compare the optimization performance of SILBO with $bottom$-$up$ strategy under different embedding updating frequency. Figure \ref{fig:unlabeledeffectiveness} shows that the optimization performance is improved when using low updating frequency.

\begin{figure}[htb]
    \centering
    \begin{subfigure}{0.3\textwidth}
    \centering
    \includegraphics[width=2in,height =1.2in]{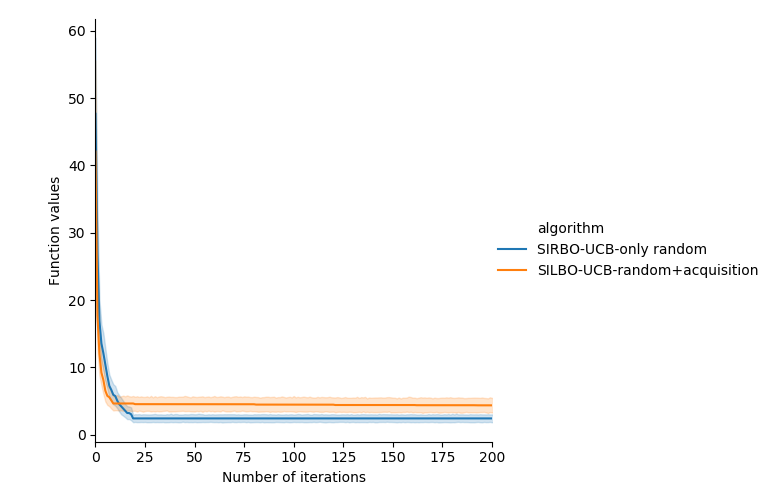}
        \caption{1000d Branin}
    \end{subfigure}\hfil
    \hspace{1em}
    \begin{subfigure}{0.3\textwidth}
    \centering
        \includegraphics[width=2in,height =1.2in]{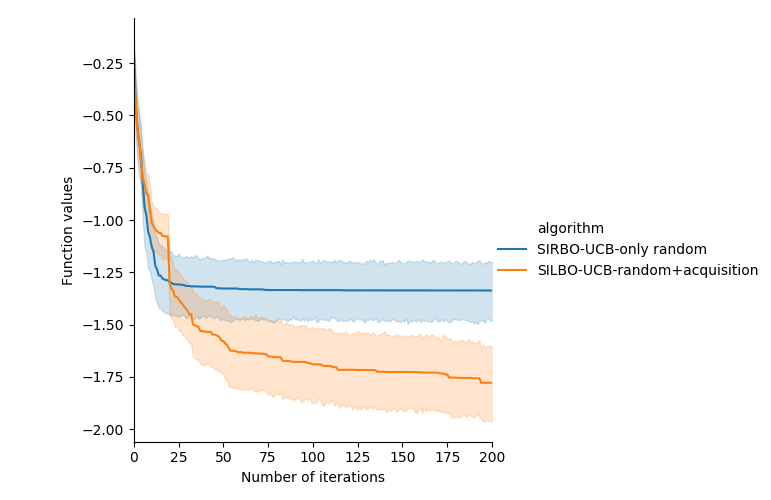}
        \caption{1000d Hartmann6}
    \end{subfigure}\hfil
    \hspace{1em}
    \begin{subfigure}{0.3\textwidth}
    \centering
        \includegraphics[width=2in,height =1.2in]{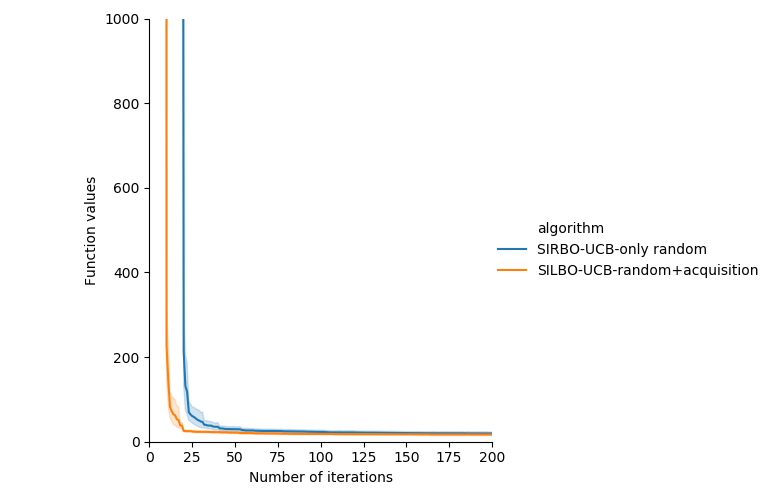}
        \caption{1000d Colville}
    \end{subfigure}\hfil
    \caption{The optimization performance of SILBO-only-random and SILBO-random+acquisition using 20 points to find the low-dimensional embedding space. The total number of points to construct the projection matrix is the same.}
    \label{fig:iterativeeffectiveness}
\end{figure}

\begin{figure}[htb]
    \centering
    \begin{subfigure}{0.3\textwidth}
    \centering
    \includegraphics[width=2in,height =1.4in]{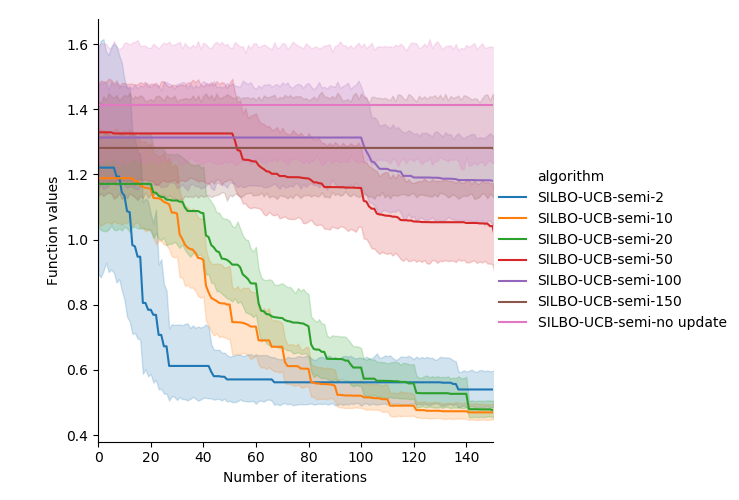}
        \caption{1000d Branin}
    \end{subfigure}\hfil
    ~
    \begin{subfigure}{0.3\textwidth}
    \centering
        \includegraphics[width=2in,height =1.4in]{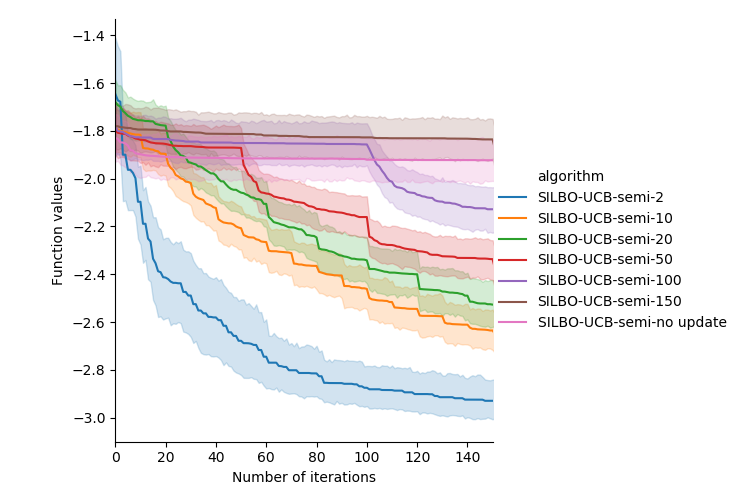}
        \caption{1000d Hartmann6}
    \end{subfigure}\hfil
    ~
    \begin{subfigure}{0.3\textwidth}
    \centering
        \includegraphics[width=2in,height =1.4in]{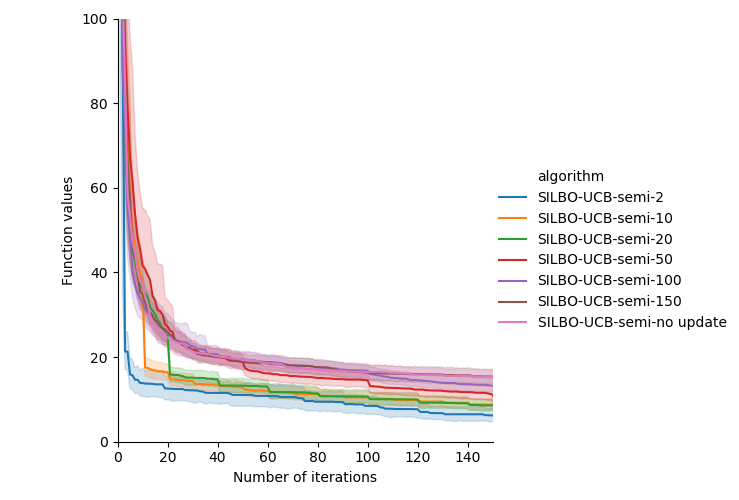}
        \caption{1000d Colville}
    \end{subfigure}\hfil
    \caption{The optimization performance of SILBO under different iterative frequency. The number represents the frequency of updating the projection matrix.}
    \label{fig:unlabeledeffectiveness}
\end{figure}

\begin{figure}[H]
    \centering
    \begin{subfigure}[t]{0.4\textwidth}
        \centering
        \includegraphics[width=0.8\textwidth]{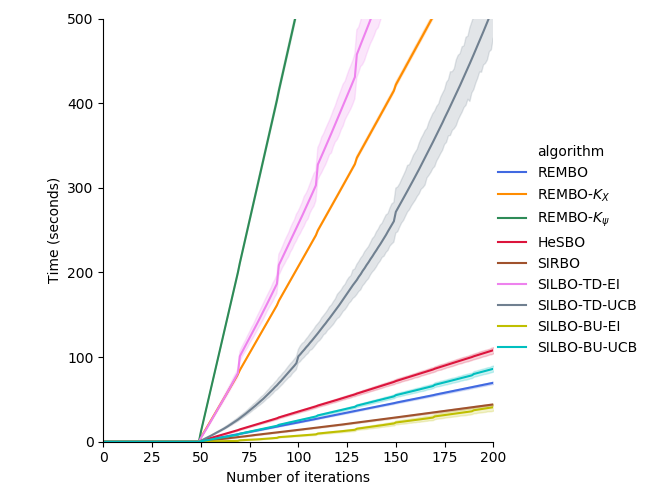}
        \caption{1000d Branin }
    \end{subfigure}%
    ~ 
    \begin{subfigure}[t]{0.4\textwidth}
        \centering
        \includegraphics[width=0.8\textwidth]{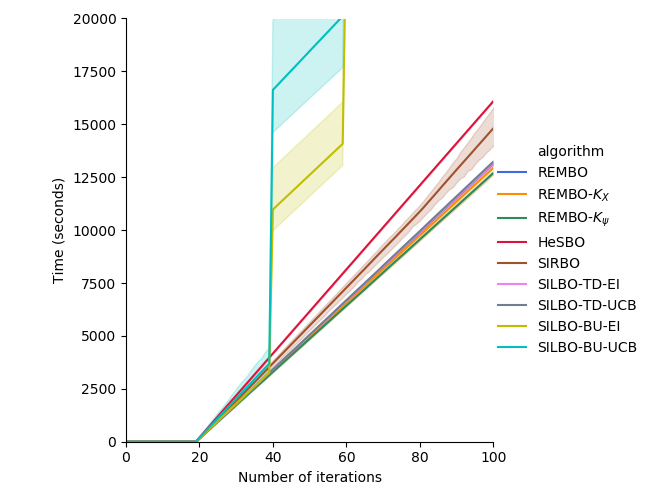}
        \caption{500d MLP}
    \end{subfigure}
    \caption{Comparison of the cumulative time on Branin with dimension 1000 and on neural network~(MLP) with dimension 500.}
    \label{fig:scalability}
\end{figure}

%%%%%%%%%%%%%%%%%%%%%%%%%%%%%%%
\subsection{Scalability Analysis}
We analyzed the scalability by comparing the cumulative time of each algorithm under the same number of iterations. As shown in Figure \ref{fig:scalability}, for low-cost objective functions such as Branin, $\text{SILBO-BU}$ is fast, while $\text{SILBO-TD}$ is relatively slow. This is because $\text{SILBO-TD}$ takes a lot of time to solve the Equation \ref{strategy2}. For expensive objective functions such as neural network, $\text{SILBO-TD}$ takes approximately the same time as most algorithms, while $\text{SILBO-BU}$ takes more time due to its re-evaluation procedure.

%conclusion
\section{Conclusion and Future Work}\label{conclusion}
High-dimensional Bayesian optimization is a very challenging task. To address the problem , we proposed a novel iterative embedding learning framework SILBO for high-dimensional Bayesian optimization through semi-supervised dimensional reduction. We also proposed a randomized fast algorithm for solving the embedding matrix efficiently. Moreover, according to the cost of the objective function, two different mapping strategies are proposed. Experimental results on both synthetic function and hyperparameter optimization tasks reveal that SILBO outperforms the existing state-of-the-art high-dimensional Bayesian optimization methods. 

In the future, we plan to address the problem of how to determine the true dimensionality during the dimension reduction process. Moreover, we also plan to use the unlabeled information to determine the kernel hyperparameters of Gaussian process and other hyperparameters of Bayesian optimization. Also, we further plan to apply SILBO to more AutoML~(Automatic Machine Learning) tasks.

% Acknowledgements should go at the end, before appendices and references

\acks{This work was supported by the National Natural Science Foundation of China~\\(U1811461, 61702254), National Key R\&D Program of China~(2019YFC1711000), Jiangsu Province Science and Technology Program~(BE2017155), National Natural Science Foundation of Jiangsu Province~(BK20170651), and Collaborative Innovation Center of Novel Software Technology and Industrialization.}


\newpage

% Note: in this sample, the section number is hard-coded in. Following
% proper LaTeX conventions, it should properly be coded as a reference:

%In this appendix we prove the following theorem from
%Section~\ref{sec:textree-generalization}:

\vskip 0.2in
\bibliography{sample}

\end{document}